\documentclass{article}

\usepackage{microtype}
\usepackage{graphicx}
\usepackage{subfigure}
\usepackage{booktabs} 
\usepackage{hyperref}

\usepackage{fullpage}

\usepackage{amsthm}

\usepackage[english]{babel}
\usepackage[utf8x]{inputenc}
\usepackage[T1]{fontenc}
\usepackage{amsmath}
\usepackage{amssymb}
\usepackage{amsfonts}
\usepackage{multirow}
\usepackage{mathabx}
\usepackage{bbm}
\usepackage{bm}
\usepackage{dsfont}
\usepackage{graphicx}
\usepackage{natbib}
\usepackage{cases}
\usepackage{algorithm}
\usepackage{algorithmic}
\usepackage[colorinlistoftodos]{todonotes}

\renewcommand{\hat}{\widehat}
\def\shownotes{1}  
\ifnum\shownotes=1
\newcommand{\authnote}[2]{[#1: #2]}
\else
\newcommand{\authnote}[2]{}
\fi

\renewcommand{\hat}{\widehat}

\newtheorem{theorem}{Theorem}[section]
\newtheorem{lemma}{Lemma}[section]

\newtheorem{claim}{Claim}[section]
\newtheorem*{remark*}{Remark}

\newtheorem*{observation*}{Observation}

\numberwithin{equation}{section}

\newcommand{\E}{\mathbb{E}}
\newcommand{\cov}{\textup{Cov}}
\newcommand{\R}{\mathbb{R}}

\newcommand{\cF}{\mathcal{F}}

\newcommand{\cL}{\mathcal{L}}

\newcommand{\cN}{\mathcal{N}}
\newcommand{\cS}{\mathcal{S}}
\newcommand{\cU}{\mathcal{U}}
\newcommand{\cV}{\mathcal{V}}
\newcommand{\cW}{\mathcal{W}}

\newcommand{\argmin}{\arg \min}
\newcommand{\argmax}{\arg \max}
\newcommand{\ce}{\loss^{\textup{ce}}}
\newcommand{\tr}{\textup{tr}}
\newcommand{\onevec}{\vec{1}}
\newcommand{\zerovec}{\vec{0}}
\newcommand{\diag}{\textup{diag}}
\newcommand{\probs}{p}
\newcommand{\droppr}{q}
\newcommand{\sftmax}{\textup{softmax}}
\newcommand{\deriv}{D}
 \newcommand{\var}[1]{\bm{#1}}
\newcommand{\hidden}[1]{h_{#1}} \newcommand{\net}{F}
\newcommand{\neti}[1]{\net_{#1}} \newcommand{\drop}[1]{\textsc{Dropout}_{#1}}
\newcommand{\elldrop}{\ell_{\textup{drop}}}
\newcommand{\helldrop}[1]{\hat{\ell}_{\textup{drop}, #1}}
\newcommand{\lhess}{H_{\textup{out}}}
\newcommand{\jac}[1]{J_{\net, {#1}}} \newcommand{\ljac}[1]{J_{\textup{loss}, {#1}}} \newcommand{\noise}{\eta} \newcommand{\mask}[1]{\noise_{#1}} \newcommand{\perturb}{\delta}
\newcommand{\expdrop}{L_{\textup{drop}}}
\newcommand{\nodrop}{L}\newcommand{\regdrop}{R_{\textup{drop}}}
\newcommand{\impdrop}{\xi_{\textup{drop}}} \newcommand{\wimpdrop}{\widetilde{\xi}_{\textup{drop}}}
\newcommand{\regexp}[1]{R_{\textup{approx}#1}}
\newcommand{\reglossj}{\widetilde{R}_{\textup{approx}}}
\newcommand{\impours}[1]{\xi_{\textup{approx}#1}} 
\newcommand{\matip}[2]{\left \langle {#1}, {#2} \right \rangle} 

\newcommand{\dout}{c} \newcommand{\din}{d}
\newcommand{\loss}{\ell} \newcommand{\ub}{B} \newcommand{\radius}{\sigma} \newcommand{\ploss}{\widetilde{\loss}} \newcommand{\gnorm}[1]{\| #1 \|_2}
\newcommand{\normto}[1]{\| #1 \|_{2, 1}}

\newcommand{\emplclass}{\hat{\cL}}

\newcommand{\pnlty}{s}
\newcommand{\dist}{P}
\newcommand{\edist}{P_n}
\newcommand{\erad}{\hat{\mathfrak{R}}}
\newcommand{\cover}{\cN}
\newcommand{\const}{\rho}
\newcommand{\gradm}{\mu}
\newcommand{\hessm}{\nu}

\title{The Implicit and Explicit Regularization Effects of Dropout}

\author{Colin Wei\thanks{Stanford University, email:
    colinwei@stanford.edu}~ \and Sham Kakade\thanks{Microsoft Research
    \& University of Washington, email: sham@cs.washington.edu}~ \and Tengyu Ma\thanks{Stanford University, email: tengyuma@stanford.edu}}
\date{March 3, 2020}
\begin{document}
\maketitle

\begin{abstract}
  Dropout is a widely-used regularization technique, often required to
  obtain state-of-the-art for a number of architectures. This work
  demonstrates that dropout introduces two distinct but entangled
  regularization effects: an \textit{explicit} effect (also studied in
  prior work) which occurs since dropout modifies the expected 
  training objective, and, perhaps surprisingly, an additional
  \textit{implicit} effect from the stochasticity in the dropout
  training update. This implicit regularization effect is 
  analogous to the effect of stochasticity in small mini-batch
  stochastic gradient descent.  We disentangle these two effects through 
  controlled experiments. We then derive analytic simplifications which
  characterize each effect in terms of the derivatives of the model
  and the loss, for deep neural networks.  We demonstrate these
  simplified, analytic regularizers accurately capture the important
  aspects of dropout, showing they faithfully replace dropout in
  practice.
\end{abstract}
\section{Introduction}
Dropout is a commonly used regularization technique for neural nets~\citep{hinton2012improving,srivastava2014dropout}. In NLP, dropout is the norm on both small and large models, as it is much more effective than methods such as $\ell_2$ regularization~\citep{merity2017regularizing}. 
In vision, dropout is often used to train extremely large models such as EfficientNet-B7~\citep{tan2019efficientnet}.

At training time, dropout sets a random subset of activations to zero,
perturbing the network output with a remarkable amount of
noise. Testing is performed on the full model, and it is somewhat
mysterious that dropout works so well despite this difference between
train and test. The esoteric nature of dropout has inspired a large
body of work studying its regularization effects:
\citet{wager2013dropout,helmbold2015inductive,cavazza2017dropout,mianjy2018implicit,mianjy2019dropout}~study
dropout for linear models, matrix factorization, and linearized
networks; ~\citet{arora2020dropout} study deep networks with dropout
only at the last layer. These works primarily study simpler settings
than those used in practice, and, as we demonstrate, there is an
\textit{implicit} regularization effect of dropout that is not
adressed by prior work.

A large body of recent work has studied implicit, or algorithmic
regularization in deep learning, defined to be a regularization effect
imposed by the training algorithm, not by the objective (see for
example~\citep{gunasekar2017implicit,li2017algorithmic,gunasekar2018characterizing,arora2019implicit}
and references therein). One notable example of this is in comparing
the generalization performance of SGD vs GD: the implicit
regularization effect of stochasticity in SGD has been empirically
studied in the context of small v.s. large batch
training~\citet{keskar2016large}, where it is observed that noisier
small-batch SGD converges to ``flatter'' local minima which generalize
better, whereas large-batch SGD converges ``sharper'' local minima
which generalize more poorly. The starting point of this work is
observing that in practice, dropout also introduces an implicit source
of regularization because it adds noise to the gradient updates
(somewhat analogous to the small v.s. large batch
training). Prior studies of dropout only analyze its \textit{explicit} regularization effect, focusing on how it modifies the expected loss.\footnote{Prior work~\citep{mianjy2018implicit} refers to this as the ``implicit bias'' of dropout. We refer to this as explicit regularization and reserve the term ``implicit'' to mean algorithmic regularization effect which does not change the objective.} 
Understanding dropout in practical settings requires studying both
regularization effects.

This paper focuses on a sharp characterization of the regularization
effects in dropout, where we:
\begin{itemize}
	\item  disentangle and analytically characterize the explicit and implicit regularization effects of dropout.
	\item derive simplified, analytical, and interpretable regularizers which completely replace dropout for language modeling tasks.
\end{itemize} 

More concretely, this work makes the following contributions: 

1. This work empirically shows that dropout provides both explicit and
implicit regularization effects. Dropout modifies the expected
training objective, and it is natural to define the  explicit
regularizer as the difference between the expected training
objective and the standard objective,  as follows:
\begin{align}
	R(\net) = \E_{x} \left[\E_{\textup{drop}}[\loss(\net_{\textup{drop}}(x))]\right] - \E_x\left[\loss(\net(x))\right] \nonumber
\end{align}
Here $F_{\textup{drop}}$ denotes the dropout model and $\textup{drop}$ denotes the randomness from dropout. 
Moreover, the optimization uses a stochastic approximation of the expected training loss by sampling the dropout noise, which gives rise to an implicit regularization effect. 

In practice, the two regularization effects are entangled and easy to
conflate. Section~\ref{sec:disentangle} provides results of
experiments which disentangle these effects.  

2. We then distill these two regularization effects, providing
simpler and more interpretable regularizers that depend on the
derivatives of the model and loss (Section~\ref{sec:analytic}). 
Intuitively, dropout regularizes the stability of the model and loss output evaluated on each training datapoint. 
Theoretically (in Section~\ref{sec:theory}), we
provide a generalization bound which helps justify the dependencies of
these regularizers on the loss derivatives. 

3. Empirically, detailed experiments are provided in Section~\ref{sec:experiments} showing that these
simplified, analytical regularizers can faithfully match and
replace dropout for both LSTM and Transformer architectures, on the
Penn Treebank, Wikitext-2, and Wikitext-103 datasets. To our
knowledge, these are the most accurate empirical demonstrations of
theory matching practice with regards to the analysis of dropout.\footnote{Our code is available at \url{https://github.com/cwein3/dropout-analytical}.}

4. Finally, the form of the derived explicit
regularizer provides detailed intuition on how to regularize
the stability of a deep model. When the number of output classes
(i.e. vocabulary in language modeling) is large, dropout regularizes
most heavily the stability of predictions corresponding to classes to which the model assigns
a prediction probability that is not too certain (i.e., not close to
either 0 or 1). Our ablation experiments in Section~\ref{sec:hessian}
reveal this is critical for the
effectiveness of dropout, and our theory in Section~\ref{sec:theory}
offers additional justification for this perspective.

More generally, we hope that the precise
methodological derivations that we provide can inform the future study and derivation
of data-dependent regularizers in deep learning.

\section{Preliminaries}\label{sec:prelim}
 \textup{\noindent {\bf Notation.}} For a function $f(a) : \R^{d_1} \to \R^{d_2}$ on a vector $a$, we use $\deriv^k_{\var{a}} f[b] \in \R^{d_2 \times d_1}$ to denote the $k$-th derivative of $f$ with respect to the variable $a$ (with a bold subscript to distinguish the variable from its value) evaluated at $b$. 
 We drop the subscript $\var{a}$ when the emphasis is unnecessary. 
 Let $f \circ g$ to denote the composition of $f$ with $g$. 
 For vector $v$, let $\diag(v)$ denote the matrix with entries of $v$ on its diagonal and 0 elsewhere. Let $\tr(M)$ denote the trace of $M$. 
 For matrices $M_1, M_2$, let $\matip{M_1}{M_2}$ denote the inner product of their vectorizations. 
 For a vector $v$, we use $(v)_i$ to denote the $i$-th coordinate of $v$, dropping the parenthesis when it is clear from context. For vectors $v_1, v_2$, $v_1 \odot v_2$ refers to their entrywise product. 
 We let $v_1^{\odot 2} \triangleq v_1 \odot v_1$. Let $\onevec$ denote the all 1's vector.
 Throughout the paper, we consider a neural network $\net$ with
 weights $W$ and a loss function $\loss : \R^{\dout} \times [\dout]
 \to \R$, where $\dout$ is the number of classes. We omit the dependence on the weights $W$ when it is clear from context. 
 We use $x$ and $y$ to denote inputs and labels. 
 The loss is computed via $\loss(\net(x), y)$, where we hide $y$ when it is clear from context.  
 
 \noindent {\bf Dropout.} 
 The most common variant of dropout, node dropout~\citep{hinton2012improving,srivastava2014dropout}, randomly sets hidden activations in a given layer to 0. Formally, 
 for some probability $\droppr$ and layer $\hidden{} \in \R^d$
 (i.e. $h$ is the vector of outputs of some layer), we sample a scaling vector $\noise \in \R^d$ with
 independent random coordinates:  
 $$(\noise)_k = \begin{cases}
  -1 &\text{ with probability } \droppr\\
  \frac{\droppr}{\droppr-1} &\text{ with probability } 1 - \droppr
  \end{cases}$$
Here $k$ indexes a coordinate of $h$. Note that $\noise$ is a zero mean random variable. 
 We then apply dropout by computing 
\[
{\hidden{}}_{\textup{drop}} = (\onevec + \noise) \odot \hidden{} 
\]
and using ${\hidden{}}_{\textup{drop}}$ instead of $\hidden{}$.
With slight abuse of notation, we let $\noise$ denote the collection of such vectors over all layers. 
 $\net(x, \noise)$ denotes the output of model $\net$ on input $x$
 using dropout noise $\noise$.   

\begin{figure*}
	\centering
	\includegraphics[width=0.23\textwidth]{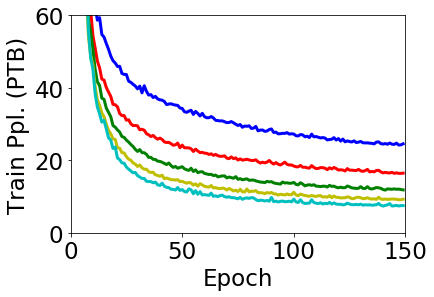}
	\includegraphics[width=0.23\textwidth]{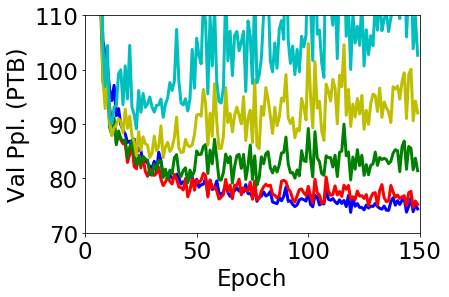}
	\includegraphics[width=0.23\textwidth]{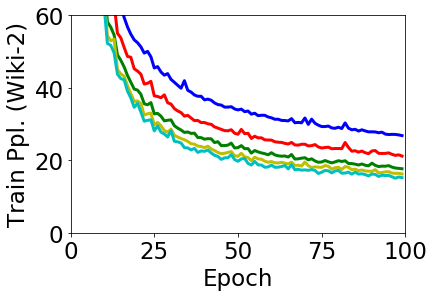}
	\includegraphics[width=0.23\textwidth]{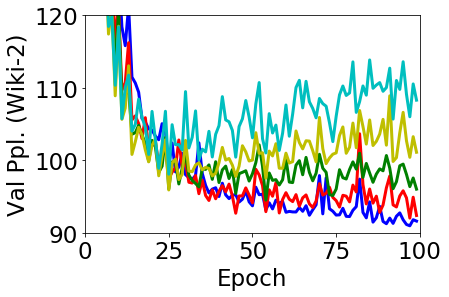}\\
	\centering \includegraphics[width=0.5\textwidth]{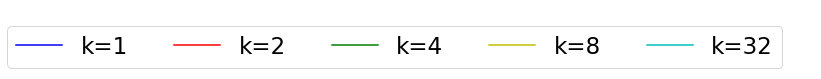}
	\caption{\textbf{Averaging dropout noise degrades performance.}
		Perplexity vs. epoch of LSTMs trained with mini-batch dropout, $\drop{k}$ for various $k$ (see Algorithm~\ref{alg:drop_k}). Training perplexity is evaluated without dropout. Increasing the number of samples of dropout noise, $k$, reduces the amount of update noise arising from the stochasticity of dropout. Though the training objective does not change with $k$, as $k$ increases, the validation performance degrades. This suggests that the update noise from dropout provides an \textit{implicit} regularization effect. \textbf{Left:} Penn Treebank. \textbf{Right:} WikiText-2. 
	}
	\label{fig:dropout_samples}
\end{figure*}

\begin{algorithm}[t]
	\caption{$\drop{k}$, mini-batch dropout update using $k$ samples of noise.}
	\label{alg:drop_k}
	\begin{algorithmic}
		\STATE {\bfseries Input:} Training examples $\{x_i\}_{i = 1}^{m}$.
		\STATE Sample noise $\noise_{ij}$ for $i \in [m], j \in [k]$.
		\STATE Compute $g = \nabla_W \left(\frac{1}{m} \sum_{i = 1}^{m} \helldrop{k}(\net, x_i, \{\noise_{ij}\}_{j = 1}^k) \right)$.
	 $\triangleright$ Use $g$ for optimization algorithm.
	\end{algorithmic}
\end{algorithm}

\section{Disentangling Explicit and Implicit Regularization in Dropout} \label{sec:disentangle}

We now present an experimental study designed to disentangle the two
regularization effects, which confirms the existence of implicit
regularization in dropout. Furthermore, this approach allows us to study each effect in
isolation.  

Let $\nodrop(F) \triangleq \E_x [\loss(\net(x))]$ denote the
population loss without dropout. This is the test criterion regardless
of whether dropout is used during training. However, dropout modifies
the \textit{expected training} objective even conditioned on a fixed
example $x$. The training loss of an example $x$ averaged over the
dropout noise $\eta$ (defined in Section~\ref{sec:prelim}) is 
\begin{align}
\elldrop(F, x) \triangleq \E_\noise [\loss(\net(x, \noise))] \ne \loss(\net(x)) \label{eq:elldrop}
\end{align}
Consequently, the expected training objective also differs from $\nodrop(F)$:
$$\expdrop(\net) \triangleq \E_{x}[\elldrop(\net, x)]  \neq \nodrop(F)$$ 

\begin{figure}
	\centering
	\includegraphics[width=0.23\textwidth]{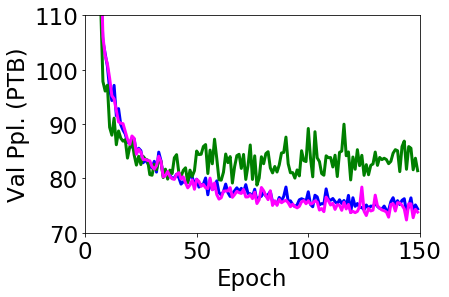}
	\includegraphics[width=0.23\textwidth]{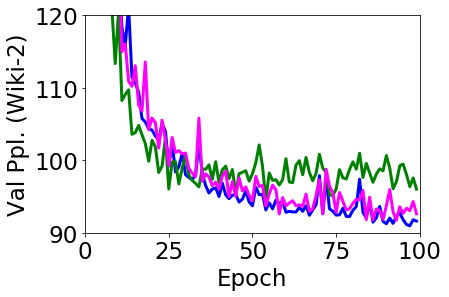}\\
	\centering\includegraphics[width=0.5\textwidth]{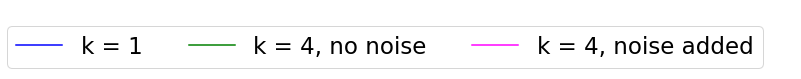}
	\caption{\textbf{Confirming implicit regularization effect.} Validation perplexity vs. epoch of LSTMs trained with $\drop{1}$, $\drop{4}$, and $\drop{4}$ with noise added via the procedure in Section~\ref{sec:drop_k_noise}. By adding noise to $\drop{4}$, we recover the performance of $\drop{1}$. Thus, the noise we add has an implicit regularization effect. \textbf{Left:} Penn Treebank. \textbf{Right:} WikiText-2.}
	\label{fig:dropout_samples_corrected}
\end{figure}

It is natural to define the  explicit
regularizer as the difference between the expected training
objective (averaged over both $x$ and $\eta$)  and the standard objective,  i.e.
$$\regdrop(F)   = \expdrop(F) - \nodrop(F).$$
Due to the fact that in practice, we only have access to a finite training
sample (and not the population), it is helpful to define explicit regularizer on a single
example as follows: 
$$\regdrop(F, x)  \triangleq \elldrop(F, x) - \ell(F(x)).$$

Previous work studies the analytical forms or properties of these
regularizers for various models. However, in practice, $\elldrop(\net,
x)$ (and its gradient $\nabla_W \elldrop(\net, x)$) are only
\textit{stochastically} estimated by sampling a single $\eta$ and
computing $\loss(\net(x, \noise))$ (and $\nabla_W \loss(\net(x,
\noise))$ respectively).  For example, SGD (with mini-batch size $1$),
performs the update:
\[
W \leftarrow W - \gamma \nabla_W \loss(\net(x, \noise)) \, ,
\]
where $\gamma$ is the stepsize, $x$ is a randomly sampled datapoint, and $\noise$ is a randomly
sampled dropout noise variable.

We demonstrate that the stochasticity from sampling $\noise$ provides an implicit regularization effect which contributes to the test-time effectiveness of dropout.\footnote{There is also an implicit regularization effect from sampling the SGD minibatch. As the minibatch size is fixed in our experiments, this is distinct from the implicit regularization effect of dropout demonstrated in Figure~\ref{fig:dropout_samples}, and studying it is orthogonal to our work.} Our strategy for disentangling the regularization effects is simple: we remove noise from the gradient estimate by optimizing a more accurate estimate of $\elldrop(F, x)$ than $\loss(\net(x, \noise))$.
Formally, we can perform ``mini-batch'' dropout by averaging the loss over $k$ samples of the noise $\{\noise_i\}_{i = 1}^k$:
\begin{align} \label{eq:drop_sample_k}
\helldrop{k}(\net, x, \{\noise_{i}\}_{i = 1}^k) \triangleq \frac{1}{k} \sum_{i = 1}^k \loss(\net(x, \eta_i))
\end{align}
For training, we now use the stochastic gradient $\nabla_W \helldrop{k}$, reducing the gradient covariance by a factor of $k$. We refer to the mini-batched dropout update by $\drop{k}$ as shorthand and formally describe it in Algorithm~\ref{alg:drop_k}. 

If there were no implicit regularization from the stochasticity of dropout, then we would expect $\drop{k}$ to have similar test performance to $\drop{1}$, which is equivalent to standard dropout. 
In Figure~\ref{fig:dropout_samples}, we plot the validation accuracy vs. training steps for models trained using $\drop{k}$ for various values of $k$. Figure~\ref{fig:dropout_samples} shows that, perhaps surprisingly, performance degrades quite sharply for larger choices of $k$. However, the explicit regularizer is still helpful, as $\drop{32}$ does not overfit as severely as the model trained without dropout (for Penn Treebank, the best perplexity without dropout is around 120, which is outside the bounds of the graph). 
$\drop{k}$ and $\drop{1}$ optimize the same expected objective, so the change in algorithm must be the cause of these performance discrepancies. 

\subsection{Injecting Dropout Noise Fixes $\drop{k}$} \label{sec:drop_k_noise}
Our proposed explanation for Figure~\ref{fig:dropout_samples} is that the gradient noise induced by dropout provides an implicit regularization effect. We verify this constructively by adding noise to the $\drop{k}$ updates in order to recover the performance of standard dropout. Let $\impdrop$ denote the fluctuation of the stochastic dropout gradient around its mean:
\begin{align} 
\impdrop(\net, x, \noise) \triangleq \nabla_W \loss(\net(x, \noise)) - \nabla_W \elldrop(\net, x)
\end{align}
Note that $\impdrop$ is exactly the gradient noise in standard dropout. Furthermore, we have $\cov(\nabla_W \helldrop{k}) = \frac{1}{k} \cov(\impdrop)$. To correct the covariance of the $\drop{k}$ gradient, we will add mean-zero noise with covariance $\left(1 -\frac{1}{k}\right) \cov(\impdrop)$. Let $\eta_1$ and $\eta_2$ be two independent draws of the dropout noise. Define $\wimpdrop$ as: 
\begin{align} \wimpdrop(\net, x, \noise_1, \noise_2) \triangleq  \nabla_W \loss(\net(x, \noise_{1})) - \nabla_W \loss(\net(x, \noise_{2}))\nonumber
\end{align} 
Note that $\cov(\wimpdrop) =  2\cov(\impdrop)$. Thus, by adding the term $\sqrt{\frac{1}{2}(1 - \frac{1}{k})} \wimpdrop$ to the $\drop{k}$ gradient, we obtain a gradient estimate with the same covariance as $\impdrop$. 

In Figure~\ref{fig:dropout_samples_corrected}, we verify that this correction procedure recovers the test performance of $\drop{1}$. Thus, we have constructed a (complicated) implicit regularizer which explains the discrepancy between $\drop{k}$ and $\drop{1}$. In Section~\ref{sec:simple_implicit}, we will explore its simplifications.

\section{Characterizing the Dropout Regularizers}
\label{sec:analytic}
Having disentangled the explicit and implicit regularization effects of dropout, we will now study them separately. 
In this section, we adapt the analysis tools of~\citep{wager2013dropout} to study both regularization effects for neural networks. 
We derive analytic simplifications for both regularizers in terms of the model and loss derivatives. At a high level, our derivations show that dropout regularizes the data-dependent stability of the model and loss on the training examples. 
This demonstrates a key difference between dropout and $\ell_2$ regularization, which is agnostic to the data. 

In Section~\ref{sec:simpleexplicit}, we present and derive our explicit regularizer. In Section~\ref{sec:simple_implicit}, we derive an update noise distribution which captures the implicit regularization effect in dropout. In Section~\ref{sec:theory}, we prove a generalization bound for the cross-entropy loss which further justifies our stability-based regularizers. In Section~\ref{sec:experiments}, we empirically demonstrate that our derivations accurately capture the regularization effects in dropout -- we can match the performance of dropout for language modeling tasks by using only our regularizers.

 For simplicity, we focus on node dropout~\citep{hinton2012improving,srivastava2014dropout}, though our analysis applies to variants such as DropConnect as well~\citep{wan2013regularization}.

\subsection{Characterizing the Explicit Regularizer} 
\label{sec:simpleexplicit}

\noindent{\bf Single-layer Dropout.} For simplicity, we start by considering node dropout applied to a single layer $i$ of the network. For the rest of the paper, we use $\hidden{i}$ to denote the $i$-th hidden layer of the network and let $\neti{i}$ denote the composition of the layers after $\hidden{i}$, that is, the function that takes in $\hidden{i}$ as input, and outputs the model prediction. (Thus, $F_i(h_i) = F(x)$). 

We rewrite the loss after applying dropout on $h_i$ by
$\loss(\net(x, \noise)) = \loss (\neti{i}(\hidden{i}(x) + \perturb)) $,
where $\perturb \triangleq \mask{i} \odot \hidden{i}(x)$ is the perturbation to the $i$-th layer. We can apply Taylor expansion to analyze the effect of this perturbation.\footnote{Taylor expansion typically requires a small level of perturbation, which may not hold if the dropout probability is large. In Section~\ref{sec:taylor_exp:app} , we argue that performing Taylor expansion around the \textit{next} layer could remedy this issue. As it does not change the final result, we omit this analysis here.} We apply Taylor expansion around $\perturb = \zerovec$:
\begin{align}
&\loss(\net(x, \noise)) - \loss(\net(x)) \approx \nonumber \\ & \deriv_{\var{\hidden{i}}} (\ell \circ \neti{i})[\hidden{i}] \perturb +  \frac{\perturb^\top (\deriv^2_{\var{\hidden{i}}} (\ell \circ \neti{i})[\hidden{i}]) \perturb}{2} \label{eq:taylor_exp}
\end{align}
This provides an approximate version of the dropout explicit regularizer $\regdrop$:
\begin{align} 
\regdrop(\net, x) &= \E_\noise [\loss(\net(x, \noise))]- \loss(\net(x)) \nonumber \\ &\approx  \frac{1}{2}\E_{\perturb} [\perturb^\top (\deriv^2_{\var{\hidden{i}}} (\ell \circ \neti{i})[\hidden{i}(x)]) \perturb] \nonumber
\end{align}
Here the expectation over the linear term in~\eqref{eq:taylor_exp} vanished because $\delta = \mask{i} \odot \hidden{i}(x)$ is a mean-zero vector. 
Next we take expectation over $\delta$: 
\begin{align}
& \E_{\perturb} [\perturb^\top \deriv^2_{\var{\hidden{i}}} (\ell \circ \neti{i})[\hidden{i}] \perturb] \nonumber \\&= \matip{\deriv^2_{\var{\hidden{i}}} (\ell \circ \neti{i})[\hidden{i}]} {\frac{\E[\perturb \perturb^\top]}{2}}\nonumber \\
& = \frac{\droppr}{2(\droppr - 1)}\matip{\deriv^2_{\var{\hidden{i}}} (\ell \circ \neti{i})[\hidden{i}]} {\diag(\hidden{i}(x)^{\odot 2})} \label{eqn:2}
\end{align}
where $\droppr$ is the dropout probability and we used the fact that $\E[\perturb \perturb^\top] = \frac{\droppr}{\droppr - 1}\diag(\hidden{i}(x)^{\odot 2})$ because $\perturb = \mask{i}\odot \hidden{i}(x)$ and the coordinates of $\mask{i}$ are independent.

We obtain an analytical approximation for the explicit regularizer $\regdrop(F,x)$ by combining the equations above. Next we will rewrite the RHS of~\eqref{eqn:2} in a more interpretable form by further dropping some terms. 

For notational simplicity, let $\jac{i}(x)$ be the Jacobians of the network output with respect to the hidden layers, and $\lhess(x)$ be the Hessian of the loss with respect to the network outputs:
\begin{align}
\jac{i}(x) \triangleq \deriv_{\var{\hidden{i}}} \neti{i} [\hidden{i}(x)] \text{ and } 
\lhess(x)  \triangleq \deriv^2_{\var{\net}} \loss [\net(x)]\nonumber
\end{align}
We claim that $\regdrop(F, x)$ (or the RHS of~\eqref{eqn:2}) can be replaced by the following analytical form
\begin{align}
\regexp{}^i(\net, x) \triangleq \matip{\jac{i}(x)^\top \lhess(x) \jac{i}(x)}{\diag(\hidden{i}(x)^{\odot 2})}\label{eqn:3}
\end{align}
Readers may find this reminiscent of the decomposition of the Hessian of neural nets loss~\cite{lecun2012efficient,sagun2017empirical}. Indeed, we decompose $\deriv^2_{\var{\hidden{i}}} (\ell \circ \neti{i})[\hidden{i}]$ into two terms, and drop the non-PSD term that depends on the Hessian of the model (which is less suitable as a regularizer and has been argued to be less important empirically~\cite{sagun2017empirical}). A full derivation and justification is given in Section~\ref{sec:analytic_hessian:app}. 

\noindent{\bf Multi-layer Dropout.} To deal with dropout on all layers, we simply take Taylor expansion with all the $\delta$'s at every layer. Cross terms cancel because the masks of different layers are independent, and the resulting regularizer is a sum of equation~\eqref{eqn:3} over $i$, giving our analytical explicit regularizer:
\begin{align} 
&\regexp{}(\net, x) \triangleq \nonumber \\ &\sum_i \matip{\jac{i}(x)^\top \lhess(x) \jac{i}(x)}{\diag(\hidden{i}(x)^{\odot 2})} \label{eq:reg_ours}
\end{align}
\noindent{\bf Interpretation. }Our regularizer ensures that the Jacobians and hidden layers of the model output are small when measured in the norm of $\lhess$.  We note that for cross entropy loss, $\lhess(x) = \diag(\probs) - \probs \probs^\top\succcurlyeq 0$, where $\probs$ is the probability vector predicted by the model encoding the distribution over output class labels. As the diagonal entries take the form $\probs_k(1 - \probs_k)$, this Hessian places stronger emphasis on output classes which the model believes are plausible but not certain. 
Our experiments in Section~\ref{sec:hessian} demonstrate that this particular weighting is an important factor for the success of dropout -- alternative ways to weight the stability of each output class in the regularizer do not perform as well. 

\citet{keskar2016large,yao2018hessian,jastrzebski2018relation} study the relationship between SGD batch size and notions of ``flatness'' of local minima via metrics related to the magnitudes of the eigenvalues of the second derivative of the loss with respect to the model parameters. They observe that flatter local minima tend to correlate with better generalization. Our regularizer encourages a notion of flatness that depends on the second derivative of the loss with respect to the \textit{hidden layers} (see~\eqref{eqn:2} in our derivation). These quantities are closely related. For example, consider weight matrix $Z$ parametrizing some linear transformation layer, such that $F(x) = F'(Z h(x))$, where $h$, $\net'$ denote the compositions of the layers before and after the application of $Z$. Then defining $h'(x) = Zh(x)$, we have 
$$
	\deriv_{\var{Z}} (\loss \circ \net)(x)[Z] = h(x) \deriv_{\var{h'}} (\loss \circ \net')[Zh(x)]
$$
Thus, the loss derivatives with respect to model parameters can be expressed in terms of those with respect to the hidden layers.

We emphasize that one benefit of $\regexp{}(\net, x)$ is that it provides an interpretable and detailed characterization of the explicit regularization effect of dropout. We hope this can help provide theoreticians and practictioners alike with precise intuitions on why dropout works, and, more broadly, how to design effective stability regularizers  in practice.
 
\subsection{Characterizing the Implicit Regularization Effect} \label{sec:simple_implicit}
In this section, we derive a gradient noise distribution which can replace the mean-zero gradient noise in dropout, $\impdrop$. 

\textup{\noindent {\bf Single Layer Dropout.}} As before, we start by considering the single-layer case. Instead of directly approximating $\impdrop$, which involves the intractable term $\nabla_W \elldrop(\net, x)$, we aim to approximate the noise $\wimpdrop$ defined in Section~\ref{sec:drop_k_noise} which we showed to be able to replace $\impdrop$.

We apply Taylor expansion to approximate $\wimpdrop$, only keeping the mean-zero linear  terms. Letting $\perturb_1 = \mask{i}^{(1)} \odot \hidden{i}(x)$ and $\perturb_2 = \mask{i}^{(2)} \odot \hidden{i}(x)$ denote two different perturbations to the $i$-th layer, 
we have\footnote{As the subscript has been used to index the layer, we use the superscript to index the different dropout noise samples.}
\begin{align}
&\wimpdrop(\net, x, \noise_i^{(1)}, \noise_i^{(2)}) \nonumber \\
&=  \nabla_W\left(\loss(\neti{i}(\hidden{i} + \perturb^{(1)})) - \loss(\neti{i}(\hidden{i} + \perturb^{(2)})) \right)\nonumber \\
&\approx \nabla_W \left(\ljac{i}(x) (\mask{i}^{(1)} - \mask{i}^{(2)}) \odot \hidden{i}(x)\right) \nonumber
\end{align}
where $\ljac{i}(x)$ denotes the Jacobian of the loss with respect to the hidden layers: 
$
	\ljac{i}(x) \triangleq \deriv_{\var{\hidden{i}}} (\loss \circ \neti{i})[\hidden{i}(x)].
$
Now we can replace the difference $\mask{i}^{(1)} - \mask{i}^{(2)}$ by $\mask{i} \sqrt{2}$, as the covariance is unchanged. 
After adjusting the scaling to match the covariance of $\impdrop$, we obtain the following analytic form for update noise:
\begin{align} \label{eq:imp_ours_i}
\impours{}^i (\net, x, \mask{i}) \triangleq \nabla_{W} \left(\ljac{i}(x) (\mask{i} \odot \hidden{i}(x))\right)
\end{align}
\textup{\noindent {\bf Multi-layer Dropout.}} To handle multi-layer dropout, we Taylor expand over all the layers, obtaining a sum of~\eqref{eq:imp_ours_i} over the layers:\footnote{To make tuning slightly simpler, we compute the noise by sampling the coordinates of $\mask{i}$ uniformly from $\{-1, +1\}$ and scaling by $\sqrt{\frac{\droppr}{\droppr - 1}}$, as this preserves the covariance.} 
{\small 
\begin{align}
	\impours{}(\net, x, \{\mask{i}\})  \triangleq \nabla_{W} \left(\sum_i \ljac{i}(x) (\mask{i} \odot \hidden{i}(x))\right) \label{eq:imp_ours}
\end{align}}To replace the implicit effect of dropout, we add the mean-zero noise $\impours{}$ to the gradients of the objective. 

\noindent{\bf Interpretation:} It is a major open question in deep learning theory to understand the regularization effects of noise~\citep{li2019towards}. For example, it is even unclear why mini-batch noise in SGD empirically helps in general. Prior works~\citep{yaida2018fluctuation,Wei2019HowNA} have (heuristically) suggested that the noise encourages the algorithm to find a solution that minimizes the trace of the covariance of the noise. 
As the covariance of $\impours{}$ is some function of $\{\ljac{i}\}, \{\hidden{i}\}$, and their gradients with respect to $W$, the induced regularizer controls some data-dependent stability of the model. Note the conceptual difference with the explicit regularizer, which multiplies the model Jacobian with the loss Hessian, whereas $\impours{}$ multiplies the model Jacobian with the loss Jacobian. 
 More precise interpretations are left for future work. 

In Section~\ref{sec:experiments}, we demonstrate that a combination of our explicit and implicit regularizer can successfully replace dropout. The general update rule which applies these regularizers in lieu of dropout is described in Algorithm~\ref{alg:update_rule}. 

\subsection{Theoretical Support for Stability-based Regularization}\label{sec:theory}
Recent works~\citep{arora2018stronger,nagarajan2019deterministic,wei2019data, wei2019improved} support our stability-based regularization by bounding generalization of the model in terms of its Jacobian norms on the training data. 
These bounds align with the Jacobian terms in the regularization~\eqref{eq:reg_ours}.  
However, they miss a crucial aspect of the regularizers derived in Section~\ref{sec:simpleexplicit} as they only consider derivatives of the model output, ignoring the \textit{loss} derivatives (the $\lhess(x)$ term in equation~\eqref{eq:reg_ours}). Though this is a subtle distinction, in Section~\ref{sec:hessian} we demonstrate that the loss derivatives are \textit{necessary} on language modeling tasks. 

In this section, we prove a new generalization bound for cross entropy loss on linear models. Our bound helps further justify the forms of our regularizers in~\eqref{eq:reg_ours} and~\eqref{eq:imp_ours}, as every term in our bound is scaled by a derivative of the loss.

Let $\ce_y$ denote the standard cross-entropy loss on $\dout$ classes with true label $y$. For linear models parameterized by weight matrix $W$, we compute the loss by
$$\ce_y(Wx) \triangleq -\log \frac{\exp((Wx)_y)}{\sum_{y'} \exp((Wx)_{y'})}$$
Let $\widebar{\ce} = \min\{\ub, \ce\}$ denote the truncation of the cross-entropy loss to some fixed bound $\ub > 0$. For matrix $M$, define the following $\normto{\cdot}$-norm of $M$: $\normto{M} \triangleq \sum_j \sqrt{\sum_i (M_{ij}^2)}$. Let $\dist$ denote the population data distribution and $\edist$ the distribution over training samples. We have the following generalization bound:\begin{theorem}\label{thm:cross_ent_gen}
	With probability $1 - \delta$ over the training examples, for all weight matrices $W$ satisfying the norm bound $\normto{W^\top } \le A$, the following holds: 
	\begin{align}
		&\E_{\dist}[\widebar{\ce_y}(Wx)] - 1.01 \E_{\edist} [\widebar{\ce_y}(Wx)] \lesssim 
		\frac{(A \gradm(W))^{\frac{2}{3}} (\theta B)^{\frac{1}{3}}}{n^{\frac{1}{3}}}  \nonumber\\ & +\frac{A \sqrt{\ub \hessm(W) \theta}}{\sqrt{n}}  + \frac{\ub A^2 \theta}{n (\log^2\left(\frac{\ub A^2 \theta}{\hessm(W) n} \right) + 1)} + \zeta \nonumber
	\end{align}Here $\gradm(W), \hessm(W)$ measure the Jacobians and Hessians of the loss and are defined by
				\begin{align}
				\gradm(W) &\triangleq \E_{\edist}[\gnorm{\deriv \ce_y[W x]}] \nonumber \\ \hessm(W) &\triangleq \E_{\edist}[\tr(\deriv^2 \ce_y[Wx])]
				\end{align}
			Additionally, we define $\theta \triangleq \log^3(n\dout) \max_i \gnorm{x_i}^2$ and $\zeta \triangleq \frac{B(\log(1/\delta) + \log \log n)}{n}$ is a low order term.
\end{theorem}
We provide the proof in Section~\ref{sec:theory:app}. Theorem~
\ref{thm:cross_ent_gen} helps justify regularizing the loss derivatives, showing that even if the weights are large, one can guarantee good generalization by ensuring that the loss Hessians and Jacobians are small. 
Note that the third term in the bound can be independent of the weight matrix norms if the loss Hessian is sufficiently small: when the data and weights are well-aligned, the third term can be as small as $O \left(\frac{\ub \log^3(nc)}{n}\right)$ (see Section~\ref{sec:theory:app_discussion}). In contrast, prior bounds for this setting contain a term scaling with some power of $\|W\|/\sqrt{n}$~\citep{kakade2009complexity,srebro2010smoothness}. This scaling suggests using $\ell_2$ regularization, which does not work for language modeling (see Table~\ref{tab:PTB_results} and~\ref{tab:wiki2_results} in Section~\ref{sec:experimental_result_app}). 

\section{Experiments} \label{sec:experiments}
In this section, we empirically confirm that our derivations in Section~\ref{sec:analytic} provide accurate characterizations of dropout. 
Our focus is on language modeling tasks using the LSTM and Transformer architectures. 

\begin{algorithm}[t]
	\caption{The general form of update for combinations of our explicit and implicit regularizers.}
	\label{alg:update_rule}
	\begin{algorithmic}
		\STATE {\bfseries Input:} minibatch $\{x_i\}_{i = 1}^{m}$, explicit regularizer $R$, gradient noise $\xi$, regularization strengths $\lambda_1, \lambda_2$.
		\STATE Compute $g = \frac{1}{m}\sum_{i = 1}^m  \nabla_W (\loss(\net(x_i)) + \lambda_1 R(\net, x_i)) $.
		\STATE Update $g = g + \frac{1}{m} \sum_{i = 1}^m \lambda_2 \xi(\net, x_i)$.\\
		$\triangleright$ Use $g$ for optimization algorithm.
	\end{algorithmic}
\end{algorithm}

\subsection{Our Derived Regularizers can Replace Dropout}\label{sec:replace_dropout}
In this section, we show that the regularizers derived in Section~\ref{sec:analytic} can replace dropout for LSTMs on language modeling tasks. We work with Penn Treebank~\citep{marcus1994penn}, a corpus of 887,521 tokens and Wikitext-2~\citep{merity2016pointer}, a corpus of 2,088,628 tokens. In Section~\ref{sec:add_exp}, we study whether our findings can also scale to larger datasets and architectures such as Transformer-XL~\citep{dai2019transformer}. 

For the experiments in this section, we base our model and code on~\cite{merity2017regularizing,merity2018analysis}. For the dropout-trained models, we use node dropout on the output, hidden, and embedding layers as well as DropConnect~\citep{wan2013regularization} on the weight matricess. We fix the dropout probability to $\droppr = 0.4$ for these experiments. To compute the update gradients for our regularizers, we follow the general rule described in Algorithm~\ref{alg:update_rule}. We specify additional hyperparameters in Section~\ref{sec:experimental_result_app}.

We study three settings described in detail below: our explicit regularizer $\regexp{}$ only, adding our noise $\impours{}$ to $\drop{k}$ updates, and combining our explicit and implicit regularizers. Tables~\ref{tab:PTB_results} and Tables~\ref{tab:wiki2_results} in Section~\ref{sec:experimental_result_app} summarize the experimental results on our regularizers for the Penn Treebank and Wikitext-2 datasets. We obtain our results \textit{without tuning}, as we use the regularization coefficient suggested in Section~\ref{sec:analytic} to match the dropout strength. The Jacobian optimization required for the analytical regularizers results in around 3x runtime slowdown compared to dropout, though we note that the analytical regularizers appear to optimize in fewer iterations (see Figure~\ref{fig:exp_imp_reg}).

\begin{figure}
	\centering
		\includegraphics[width=0.23\textwidth]{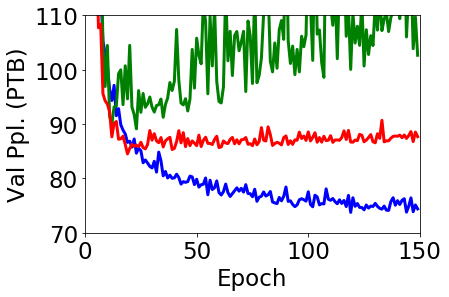}
	\includegraphics[width=0.23\textwidth]{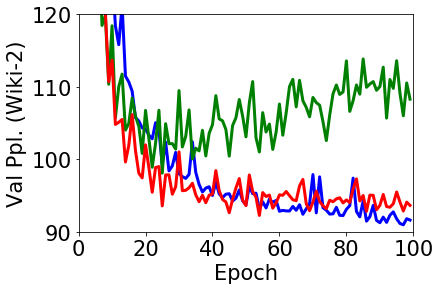}\\
	\centering \includegraphics[width=0.4\textwidth]{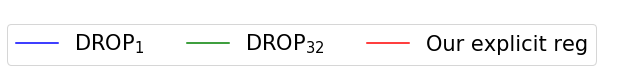}
			\caption{\textbf{Our explicit regularizer v.s. dropout.} Validation perplexity vs. epoch of LSTMs trained with $\drop{1}$, $\drop{32}$ (see Algorithm~\ref{alg:drop_k}), and our explicit regularizer only. Our explicit regularizer~\eqref{eq:reg_ours} outperforms $\drop{32}$ but does not match $\drop{1}$ since it is missing the implicit regularization benefit of dropout. \textbf{Left:} Penn Treebank. \textbf{Right:} WikiText-2.}
	\label{fig:exp_reg_only}
\end{figure}
\textup{\noindent {\bf Replacing Dropout Explicit Regularization.}} In Figure~\ref{fig:exp_reg_only}, we compare our explicit regularizer~\eqref{eq:reg_ours} to mini-batch dropout, $\drop{k}$, with $k = 1, 32$. 
For $k = 32$, the implicit regularization effect of dropout is heavily reduced, bringing the training procedure closer to training on $\elldrop$ exactly. Our explicit regularizer outperforms $\drop{32}$, confirming that it matches the explicit regularization effect of dropout. It does not match the performance of $\drop{1}$ because it is missing the implicit regularization effect.

\begin{figure}
		\centering
		\includegraphics[width=0.23\textwidth]{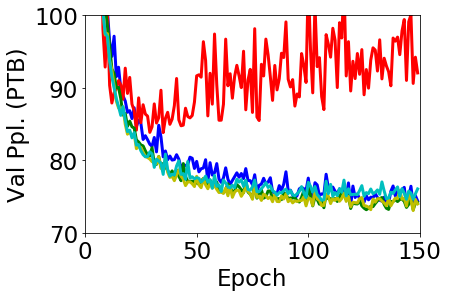}
	\includegraphics[width=0.23\textwidth]{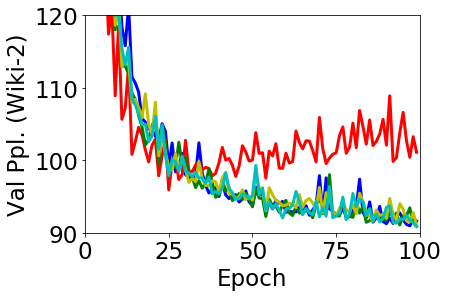}\\
	\centering \includegraphics[width=0.45\textwidth]{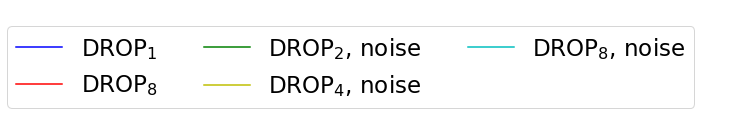}
			\caption{\textbf{Our implicit regularizer v.s. dropout.} Validation perplexity vs. epoch of LSTMs trained using mini-batch dropout, $\drop{k}$, with injection of noise $\impours{}$ (Algorithm~\ref{alg:drop_k_our_noise}). For reference, we also plot $\drop{1}$ and $\drop{8}$ with no noise  (Algorithm~\ref{alg:drop_k}). $\drop{k}$ with noise injection matches $\drop{1}$ for $k = 2, 4, 8$, affirming that our noise distribution $\impours{}$ captures the implicit regularization effect of dropout noise. \textbf{Left:} Penn Treebank. \textbf{Right:} WikiText-2.}
	\label{fig:imp_reg_only}
\end{figure}

\textup{\noindent {\bf Replacing Dropout Implicit Regularization.}} We demonstrate that the our update noise derived in~\eqref{eq:imp_ours} can effectively replicate the implicit regularization effect of dropout. We inject appropriately scaled $\impours{}$ noise into the $\drop{k}$ training procedure. As the covariance of  $\nabla_W \helldrop{k}$ scales with $\frac{1}{k}$, we scale $\impours{}$ by a factor $\sqrt{1 - \frac{1}{k}}$. Thus, if $\impours{}$ and $\impdrop$ had the same covariance, the covariance of the updates would remain constant across $k$. Algorithm~\ref{alg:drop_k_our_noise} in Section~\ref{sec:implementation_app} formally describes this procedure. In Figure~\ref{fig:imp_reg_only}, we demonstrate that this procedure can closely track the performance of $\drop{1}$ for various values of $k$, affirming that $\impours{}$ captures essential properties of $\impdrop$.

\textup{\noindent {\bf Completely Replacing Dropout.}} We demonstrate that the combination of our regularizers can completely replace dropout. We apply algorithm~\ref{alg:update_rule}, setting $R = \regexp{}$ and $\xi = \impours{}$. In Figure~\ref{fig:exp_imp_reg}, we plot the validation perplexity vs. time of a model trained with our regularization vs. $\drop{1}$. Figure~\ref{fig:exp_imp_reg} demonstrates that our regularization is enough to replace dropout, confirming the validity of our derivations. We note that our regularizer appears to require fewer iterations to decrease the validation perplexity. This raises the exciting possibility of designing more efficient regularizers than dropout, which we leave for future work.  
\begin{figure}
	\centering
		\includegraphics[width=0.23\textwidth]{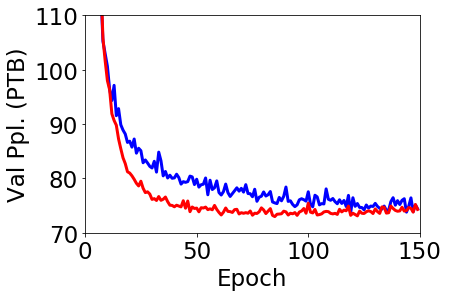}
	\includegraphics[width=0.23\textwidth]{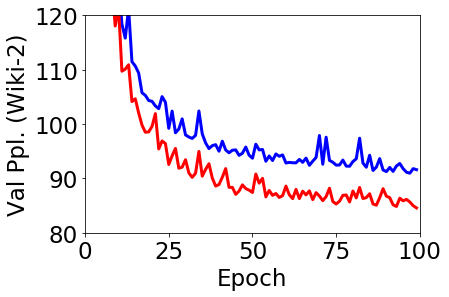}\\
	\centering \includegraphics[width=0.4\textwidth]{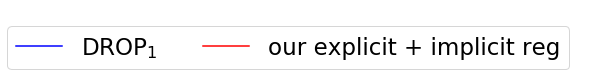}
			\caption{\textbf{Our combined regularizer v.s. dropout.} Validation perplexity vs. epoch of LSTMs trained with our regularizers vs standard dropout. Our regularizers can match dropout and appear to improve the validation perplexity faster. \textbf{Left:} Penn Treebank. \textbf{Right:} WikiText-2.}
	\label{fig:exp_imp_reg}
\end{figure}

\subsection{Regularizing the Loss Hessian is Necessary} \label{sec:hessian}
We argue that simply regularizing the stability of the model outputs is not sufficient. As argued in Section~\ref{sec:simpleexplicit}, our derivations show that dropout enforces stronger stability for output coordinates where the model assigns non-trivial probability mass but is not extremely confident. To demonstrate this is helpful, we experiment with replacing $\lhess$ in our explicit regularizer (see~\eqref{eq:reg_ours}) with two alternative quantities.

\textup{\noindent{\bf Identity Cannot Replace Loss Hessian.}}  For the first variant, we use an identity matrix instead of the loss Hessian (so the regularizer weights each output coordinate equally). We provide implementation details in Section~\ref{sec:implementation_app}. On Penn Treebank, this was ineffective: after thoroughly tuning the regularization strength, the best validation accuracy we obtained was 108.76, which is comparable to the performance of $\ell_2$ regularization and much worse than dropout. 

\textup{\noindent{\bf Using the Loss Jacobian Instead of Hessian.}} For cross entropy loss, in the case where the model predicts the true label very confidently, the loss Hessian $\lhess(x)$ approaches the outer product of the loss Jacobian with itself: $\lhess(x) \approx  \deriv_{\var{F}} \ce [\net(x)]^\top \deriv_{\var{F}} \ce [\net(x)]$ (see Section~\ref{sec:exp_reg_imp}). Substituting this approximation into our explicit regularizer gives the following alternative regularizer:
\begin{align} \label{eq:reglossj}
\reglossj(F, x) \triangleq \sum_i \ljac{i} \diag(\hidden{i}(x)^{\odot 2}) \ljac{i}^\top
\end{align}
On Penn Treebank, we find that this regularizer is much more effective than $\ell_2$ regularization but cannot match dropout. We test whether $\reglossj{}$ performs well as $\regexp{}$ with or without implicit regularization. In both cases, we tune the explicit regularization strength. Table~\ref{tab:loss_jac} summarizes the results compared to the Hessian-based regularizer. $\reglossj{}$ on its own significantly outperforms $\ell_2$ regularization and can match $\regexp{}$ after tuning. However, with update noise it does not match dropout or $\regexp{}$ (even after tuning). \begin{table}
	\centering
	\caption{Regularization effect of $\reglossj{}$ (see~\eqref{eq:reglossj}) on Penn Treebank with and without implicit regularization. $\reglossj{}$ can significantly outperform $\ell_2$ regularization but does not match dropout even with implicit regularization, whereas $\regexp{}$ can.}	\label{tab:loss_jac}
	\begin{tabular}{c c c} 
				\cline{1-3}
		&Training Method & Best Val. Ppl.\\
		\cline{1-3}
		\multirow{3}{*}{No implicit} & $\ell_2$ reg (tuned)& 112.04 \\
		&$\reglossj{}$ (tuned) & 84.06\\
		&$\regexp{}$~\eqref{eq:reg_ours}   & 84.52\\
		\cline{2-3} \addlinespace[0.1cm]
		\multirow{3}{*}{With implicit} &$\drop{1}$ &73.76\\
		&$\reglossj{}$ (tuned) and $\impours{}$& 79.54\\
		&$\regexp{}$ and $\impours{}$~\eqref{eq:imp_ours_i} & 72.99
	\end{tabular}
\end{table}

\subsection{Additional Settings} \label{sec:add_exp}
We test how well our findings translate to larger datasets and different architectures. We use the Wikitext-103 dataset, which contains 103,227,021 tokens, and the Transformer-XL~\citep{dai2019transformer} and QRNN~\citep{bradbury2016quasi} architectures. First, we explore whether the implicit regularization effect of dropout is as important on larger datasets. We train the Transformer-XL and QRNN architectures on the Wikitext-103 corpus using $\drop{k}$ for $k = 1, 2, 4$. Table~\ref{tab:wiki-103-transformer} shows that for Transformer-XL trained on the full dataset, the implicit regularization effect disappears. We observe the same for QRNN (see Section~\ref{sec:experimental_result_app}).

In Table~\ref{tab:wiki-103-transformer}, we also demonstrate that there \textit{is} an implicit regularization effect when we downsample Wikitext-103 by a factor of 5, though it is not as crucial. Thus, the importance of the implicit regularization depends on the dataset size.

Finally, we confirm that our explicit regularizer is effective on a larger dataset. For Wikitext-103 and Transformer-XL, Table~\ref{tab:wiki-103-transformer} shows that our explicit regularizer achieves validation perplexity of 24.12, within $0.7$ of dropout. 
\begin{table}
	\centering
	\caption{Experimental results on WikiText-103 dataset for Transformer architecture. The implicit regularization effect of dropout noise appears to decrease with dataset size.}	\label{tab:wiki-103-transformer}
	\begin{tabular}{c c c} 
		\cline{1-3}
		Dataset Size & Training Method & Best Val. Ppl.\\
		\cline{1-3}
		\multirow{5}{*}{Full Dataset}& No regularization & 29.45\\
		&$\drop{1}$   & 23.39\\
		&$\drop{2}$ &  23.13\\
		&$\drop{4}$ & 23.14\\
		\cline{2-3}
		& $\regexp{}$ & 24.12\\
		\cline{1-3}
		\multirow{3}{*}{$0.2\times$ Dataset} & $\drop{1}$ & 46.05\\
		& $\drop{2}$ & 46.07\\
		& $\drop{4}$ & 47.40
	\end{tabular}
\end{table}

\section{Related Work}
Dropout has been the focus of several theoretical works studying its properties for both optimization and generalization~\citep{wager2013dropout,wager2014altitude,baldi2013understanding,helmbold2015inductive,gal2016dropout,helmbold2017surprising,cavazza2017dropout,mianjy2018implicit,mianjy2019dropout,arora2020dropout}.~\citet{wang2013fast,maeda2014bayesian,gal2016dropout,ma2016dropout} study dropout from a Bayesian perspective.~\citet{gao2019demystifying} empirically study the effect of applying dropout masks in one only direction of the network (either the forward or backward pass). 

\citet{wager2013dropout,helmbold2015inductive} use a Taylor expansion to analyze dropout in linear models, and  our work extends their analysis to neural networks.~\citet{cavazza2017dropout,mianjy2018implicit,mianjy2019dropout} study the expected dropout objective for matrix factorization and linearized neural nets, respectively. Recent work~\citep{arora2020dropout}~studies dropout applied to only the last layer of a deep neural net, computing an exact expression for the explicit regularizer which depends on the magnitudes of coordinates in the last hidden layer. Our analysis of the explicit regularization results in a more general expression which contains similar quantities, but considers dropout at all layers of the network. These prior works focus on explicit regularization and do not study the implicit regularization effects of dropout. 

There has been a large body of prior work studying the relationship between gradient noise and generalization~\citep{keskar2016large,keskar2017improving,smith2017bayesian,jastrzebski2018relation,xing2018walk,li2019towards,chaudhari2018stochastic}.~\citet{jastrzkebski2017three,zhu2018anisotropic,wen2019interplay} study how the noise distribution in SGD affects generalization.~\citet{wen2019interplay} inject noise with an appropriate covariance structure into the updates of large-batch SGD, making it match the behavior of small-batch SGD. We inject appropriate noise to make large-sample dropout updates match standard dropout.

Prior works have also studied data-dependent regularizers of the model and loss stability.~\citet{sokolic2017robust,hoffman2019robust} apply Jacobian-based regularization to train robust classifiers.~\citet{krueger2015regularizing} propose a data-dependent regularizer for the stability of RNN activations.~\citet{novak2018sensitivity,arora2018stronger,nagarajan2019deterministic,wei2019data,wei2019improved} study the relationship between model stability and generalization.

Finally, regularization for deep models is an important issue in NLP.~\citet{zaremba2014recurrent} demonstrated that dropout can be very helpful for NLP tasks.~\citet{semeniuta2016recurrent,gal2016theoretically} propose variants of dropout designed for recurrent neural networks.~\citet{krueger2015regularizing,merity2017revisiting} study temporal activation stability regularization.~\citet{merity2017revisiting,melis2017state,merity2018analysis} demonstrate that the proper tuning of regularizers can greatly impact performance. 

On the broader topic of generalization theory of neural networks, ~\citet{zhang2016understanding,neyshabur2018towards} observe that deep learning defies a lot of conventional statistical wisdom. Several works have studied generalization bounds for deep networks (see~\citep{bartlett2017spectrally,neyshabur2017pac,golowich2017size,dziugaite2017computing,arora2018stronger,wei2019improved} and references therein). Another line of work studies implicit regularization in deep learning (see~\citep{gunasekar2017implicit,gunasekar2018implicit,gunasekar2018characterizing,soudry2018implicit,woodworth2019kernel,arora2019implicit} and references therein). 
	\section{Conclusion}
In this work, we show that dropout actually introduces two entangled sources of regularization: an explicit one which modifies the expected objective, and an implicit one due to stochasticity in the updates. We empirically disentangle these regularizers and derive analytic simplifications which faithfully distill each regularization effect. We demonstrate that our simplified regularizers can replace dropout in practice. Our derivations show that dropout regularizes the stability of the model and loss around the training data. 

More broadly, our analytic characterizations of dropout can provide intuition on what works and what doesn't for stability-based regularizers in deep learning. 
We hope that these intuitions can help inform and motivate the design of more principled regularizers for deep networks. 

	\section*{Acknowledgements}

We would like to thank Michael Xie for suggesting the experiment which disentangled the implicit and explicit regularization effects. Sham Kakade would also like to thank Xinyi Chen, Cyril Zhang, and
Yi Zhang for numerous helpful discussions and help with earlier experimentation.
Colin Wei acknowledges support from an NSF Graduate Research Fellowship. 
The work is also partially supported by SDSI and SAIL at
Stanford.  Sham Kakade acknowledges funding from the Washington Research Foundation for Innovation in Data-intensive Discovery, and the NSF Awards CCF-1703574, and CCF-1740551.
\bibliography{refs}
\bibliographystyle{icml2020}
\onecolumn
\appendix
\section{Full Derivations in Section~\ref{sec:analytic}}\label{sec:analytic:app}

\subsection{Full Derivation of Equation~\eqref{eqn:3}} \label{sec:analytic_hessian:app}
We decompose the second derivative of the loss with respect to $\hidden{i}$, $\deriv^2_{\var{\hidden{i}}} (\ell \circ \neti{i})[\hidden{i}(x)])$, into a positive-semidefinite component and non-PSD component:
\begin{align}\label{eq:psd_nonpsd}
\deriv^2_{\var{\hidden{i}}} \ell \circ \neti{i}[\hidden{i}(x)]) = \jac{i}(x)^\top \lhess \jac{i}(x) + M 
\end{align}
where $$M \triangleq \sum_{j \in [\dout]} (\deriv_{\var{\net}} \loss[\net(x)]))_j \deriv^2_{\var{\hidden{i}}} (\neti{i})_j [\hidden{i}(x)]$$ is some matrix capturing second derivatives of the model output. We used $(\cdot)_j$ to index the $j$-th coordinate in the output, to avoid confusion with indexing the layers. The first term in~\eqref{eq:psd_nonpsd} is positive-semidefinite when $\loss$ is convex, but $M$ is unlikely to be positive-semidefinite since it involves the Hessian of a non-convex model. Plugging everything back into~\eqref{eqn:2}, we obtain
\begin{align}
\regdrop(F, x) &\approx \frac{\droppr}{2(\droppr - 1)}\matip{\jac{i}(x)^\top \lhess \jac{i}(x)} {\diag(\hidden{i}(x)^{\odot 2})}  \\
&+ \frac{\droppr}{2(\droppr - 1)} \matip{M} {\diag(\hidden{i}(x)^{\odot 2})}
\end{align}
To ensure that our regularizer is nonnegative, we ignore the second term containing the non-PSD matrix $M$. Ignoring the non-PSD term in this kind of decomposition was also suggested in~\cite{sagun2017empirical}. We also omit the factor of $\frac{q}{2(q - 1)}$ in~\eqref{eqn:3} for simplicity. 

\subsection{Justification of Taylor Expansion} \label{sec:taylor_exp:app}
As dropout introduces a change that has magnitude which is multiplicative in the size of the coordinates of the hidden layers, the perturbation due to dropout might not be small. Since Taylor expansions typically require a small level of perturbation, in this section we argue that when the application of dropout is followed by a linear transformation layer, the perturbation to the linear layer could be small. Furthermore, we demonstrate that performing Taylor expansion with respect to this layer will ultimately give the same regularizer. 

We work in the same setting of the derivation in Section~\ref{sec:simpleexplicit}. We add the additional assumption that $h_i$ is followed by a linear transformation parameterized by weight matrix $Z$.
Thus, we can express $\net(x) = \neti{i}(\hidden{i}(x)) = \neti{i + 1}' (Z \hidden{i}(x))$ where $F_{i+1}'$ denotes all the computation after the matrix multiplication $Z$. ($\net_{i+1}'$ differs from $\net_{i+1}$ just by an additional activation layer that follows the matrix multiplication by $Z$.) Now we can compute the loss after applying dropout on $\hidden{i}$ by 
\begin{align}
	\loss(\net(x, \noise)) = \loss(\net_{i + 1}'(Z \hidden{i}(x) + Z \delta))
\end{align}

Our key observation, as detailed below, is that although the perturbation $\delta$ to $\hidden{i}(x)$ could be large relative to the magnitudes of the coordinates of $\hidden{i}(x)$, the perturbation $Z \delta$ may be much smaller relative to the magnitudes of the coordinates of $Z\hidden{i}(x)$. Thus, the effect of the dropout noise can be mitigated as it passes through linear layers of the network, making the Taylor expansion more realistic. 

Concretely, consider the standard deviation of the $j$-th coordinate of $Z \delta$: 
\begin{align}\label{eq:Mdelta_std}
 \textup{std}[(Z \delta)_j] = \sqrt{\sum_{k} Z_{jk}^2 (\hidden{i}(x))_k^2}
\end{align}
On the other hand, we have 
\begin{align} \label{eq:Mhidden}
	(Z \hidden{i}(x))_j = \sum_k Z_{jk} \cdot (\hidden{i}(x))_k
\end{align}
In the case where $\{Z_{jk}\}_k$ and $\hidden{i}$ share the same sign on each coordinate, the signal $(Z \hidden{i}(x))_j$~\eqref{eq:Mhidden} can be larger than the size of the perturbation, that is, $\textup{std}[(Z \delta)_j]$~\eqref{eq:Mdelta_std}, by a factor of $\sqrt{d}$. For example, consider the case where all entries of $Z$ and $\hidden{i}$ are 1. In other words, even though the vector $\delta$ seems to be comparable to $h_i$ in the norm, after passing through the linear transformation, due to the cancellation arising from the randomness in $\delta$, $Z\delta$ can be much smaller compared to $Zh_i$. 

Thus, when the weight matrix and hidden layer are well-aligned, the level of perturbation caused by dropout to the subsequent linear layer might not be too large. This supports our use of Taylor expansion. Now we can also check that Taylor expanding around $Z\delta = 0$ gives the same regularizers. (This is unsurprising because the form of Taylor expansion is invariant to linear transformation.) Using $\hidden{i + 1}'$ to refer to $Z\hidden{i}$, we have
\begin{align}\label{eq:taylor_exp_Mdelta}
	\loss(\net(x, \noise)) \approx \loss (\net(x)) + \deriv_{\var{\hidden{i + 1}'}}( \loss \circ \net'_{i + 1})[Z \hidden{i}(x)] Z \delta + \frac{\delta^\top Z^\top \deriv^2_{\var{\hidden{i + 1}'}} (\loss \circ \net'_{i + 1}) [Z \hidden{i}(x)] Z \delta }{2}
\end{align}
Now we observe that $$\deriv_{\var{\hidden{i + 1}'}}( \loss \circ \net'_{i + 1})[Z \hidden{i}(x)] Z = \deriv_{\var{\hidden{i}}} (\loss \circ \neti{i}) [\hidden{i}(x)]$$
$$Z^\top \deriv^2_{\var{\hidden{i + 1}'}} (\loss \circ \net'_{i + 1}) [Z \hidden{i}(x)] Z = \deriv^2_{\var{\hidden{i}}} (\loss \circ \net_{i}) [\hidden{i}(x)]$$
Substituting these back into~\eqref{eq:taylor_exp_Mdelta} brings us back to~\eqref{eq:taylor_exp}, which served as the starting point for the derivations of our explicit and implicit regularizers. Thus, we obtain the same analytic expressions by performing Taylor expansion around $Z\delta$. 

To provide further justification, we evaluate the quality of the Taylor expansion for an LSTM trained with hidden layer dropout with probability 0.5 on Penn Treebank, and found that the quadratic term accounted for > 80\% of the difference between the losses with and without dropout. In more details, we found that $\E_x[\regdrop(x) - \regexp(x)] = 0.009$, whereas $\E_x[\regdrop(x)] = 0.053$ for this particular model, where the expectation was taken over the training set. Thus, the Taylor approximation is sufficiently tight for~\eqref{eq:reg_ours} to capture the important explicit regularization effects of dropout, as is supported in Section~\ref{sec:experiments}.

\section{Proof of Theorem~\ref{thm:cross_ent_gen}}\label{sec:theory:app}
In this section, we will analyze a general loss function $\loss : \R^\dout \times [\dout] \to \R$ (note that we will frequently hide the dependence on the label $y$, as it is not important for our proofs). As before, for some fixed bound $\ub > 0$, define $\widebar{\loss} \triangleq \min\{B, \loss\}$ to be the truncation of the loss. We carry over the remainder of the notation from the setting in Section~\ref{sec:theory}. 

Our proof of Theorem~\ref{thm:cross_ent_gen} will rely on the following slightly more general statement for loss functions with an exponential tail. 
\begin{theorem}\label{thm:exp_tailed_gen}
	Suppose $\loss(\cdot, y)$ is convex and satisfies 
	\begin{align}
	\gnorm{\deriv (\tr \circ \deriv^2 \loss (\cdot, y)  )[h]} \le \tau \tr (\deriv^2 \loss (\cdot, y) [h]) \label{eq:exp_tail_cond}
	\end{align} for all $h$, $y$, and some $\tau > 0$. 
		With probability $1 - \delta$ over the draw of the training examples, for all $W \in \R^{\dout \times \din}$ satisfying the norm bound $\normto{W^\top} \le A$ the following holds: 
	\begin{align}
	\E_\dist[\widebar{\loss}(Wx, y)] - 1.01 \E_{\edist}[\widebar{\loss}(Wx, y)] \lesssim \frac{(A\gradm(W))^{2/3}(\theta B)^{1/3}}{n^{1/3}} + \frac{A \sqrt{\theta \hessm(W) B}}{\sqrt{n}} + \\
	\frac{B A^2 \theta  \tau^2}{n \left( \log^2 \left( \frac{BA^2 \theta  \tau^4}{n \hessm(W)}\right) + 1\right)} +  \frac{B(\log 1/\delta + \log \log n)}{n}
	\end{align}
	where $\theta \triangleq \max_i \|x_i\|^2 \log^3(nc)$ and $\gradm$, $\hessm$ measure the Jacobians and Hessians of the loss and are defined by 
	\begin{align}
		\gradm(W) \triangleq  \frac{\sum_{i = 1}^n \gnorm{D \loss (\cdot, y) [Wx_i]} }{n} \label{eq:gradm_def}\\
		\hessm(W) \triangleq \frac{\sum_{i = 1}^n \tr (D^2 \loss (\cdot, y) [Wx_i])}{n} \label{eq:hessm_def}
	\end{align}
\end{theorem}

Given Theorem~\ref{thm:exp_tailed_gen}, we can complete the proof of Theorem~\ref{thm:cross_ent_gen} by observing that $\ce$ satisfies~\eqref{eq:exp_tail_cond}, as formally stated in the following lemma.

\begin{lemma} \label{lem:cross_entropy_exp_tailed}
	Let $\ce$ denote the cross entropy loss. Then for any $h, y$, $\ce$ satisfies~\eqref{eq:exp_tail_cond} with $\tau = \sqrt{2}$. 
\end{lemma}
We provide the full proof of Lemma~\ref{lem:cross_entropy_exp_tailed} in Section~\ref{sec:proof_additional}. 

\begin{proof}[Proof of Theorem~\ref{thm:cross_ent_gen}]
	Using Lemma~\ref{lem:cross_entropy_exp_tailed} to observe that $\ce$ satisfies~\eqref{eq:exp_tail_cond} with $\tau = \sqrt{2}$, we can plug this value of $\tau$ into Theorem~\ref{thm:exp_tailed_gen} to get the desired result.  
\end{proof}

Thus, it suffices to prove Theorem~\ref{thm:exp_tailed_gen}. To do so, we will rely on the following lemmas. 

\begin{lemma}\label{lem:max_sigma_bound}
	In the setting of Theorem~\ref{thm:exp_tailed_gen}, define $\kappa \triangleq \max_i \gnorm{x_i}$. With probability $1 - \delta$ over the draw of the sample $\{(x_i, y_i)\}_{i = 1}^n$, for all $W \in \R^{\dout \times \din}, \normto{W^\top} \le A$ and $\alpha > 0$, the following holds:
	\begin{align} \label{eq:max_sigma_bound:1}
	\E_\dist[\widebar{\loss}(Wx, y)] \le (1 + 1/\alpha) \E_{\edist}[\widebar{\loss}(Wx, y)] + \const' \left(\min_{\beta > 0} G_{W,\alpha}(\beta) + (1 + \alpha) \frac{\ub(\log 1/\delta + \log \log n)}{n}\right)
	\end{align}
	where $G_{W, \alpha}$ is the data-dependent function defined by
	\begin{align}\label{eq:max_sigma_bound:3}
	G_{W,\alpha}(\beta) \triangleq \beta (1 + 1/\alpha)\gradm(W) \kappa  +(1 + 1/\alpha) \hessm(W)\frac{\exp(\tau \beta \kappa) - \beta \kappa - 1}{\tau^2} + (1 + \alpha) \ub \frac{A^2 \log^3(nc)}{\beta^2 n} 
	\end{align}
	where $\gradm(W), \hessm(W)$ are defined as in~\eqref{eq:gradm_def} and~\eqref{eq:hessm_def} which (implicitly) depend on the training data, 	and $\const' > 0$ is some universal constant.
\end{lemma}
We prove this lemma in Section~\ref{sec:max_sigma_bound_proof}. 

\begin{lemma}\label{lem:G_upper_bound}
	In the setting of Lemma~\ref{lem:max_sigma_bound}, let $G_{W,\alpha}(\beta)$ be defined as in~\eqref{eq:max_sigma_bound:3}. Define $\theta  \triangleq \log^3 (nc) \kappa^2$. Then 
	\begin{align}
	\min_{\beta > 0} G_{W,\alpha}(\beta) \lesssim \frac{(1 + 1/\alpha)^{2/3}(1 + \alpha)^{1/3} (\theta B)^{1/3}(A\mu(W))^{2/3}}{n^{1/3}} + \\
	\frac{A\sqrt{\hessm(W)(1 + \alpha)(1 + 1/\alpha) \theta \ub}}{\sqrt{n}}+  (1 + \alpha)\frac{\ub A^2 \theta \tau^2}{n\left(\log^2 \left(\frac{(1 + \alpha)\ub A^2\theta \tau^4}{n(1 + 1/\alpha) \hessm(W)} \right) + 1\right)}
	\end{align}
	where $\gradm, \hessm$ are defined in~\eqref{eq:gradm_def} and~\eqref{eq:hessm_def}. 
\end{lemma}
We prove this Lemma in Section~\ref{sec:G_upper_bound_proof}. We now can prove Theorem~\ref{thm:exp_tailed_gen} by combining the lemmas above.

\begin{proof}[Proof of Theorem~\ref{thm:exp_tailed_gen}]
	Combining Lemmas~\ref{lem:max_sigma_bound} and Lemma~\ref{lem:G_upper_bound} and choosing $\alpha = 100$, we get the desired result. 
\end{proof}

\subsection{Proof of Lemma~\ref{lem:max_sigma_bound}} \label{sec:max_sigma_bound_proof}
In this section, we derive the proof of Lemma~\ref{lem:max_sigma_bound}. Our proof bounds the generalization of a \textit{perturbed} loss function. By trading off between perturbation level and generalization error, we obtain Lemma~\ref{lem:max_sigma_bound}. 
Define the following perturbed version of the loss $\loss$:
\begin{align}
\ploss_{\radius}(W,x,y) = \max_{\delta \in \R^\dout} \pnlty_\radius(\gnorm{\delta}) \widebar{\loss}(Wx + \delta \gnorm{x},y)
\end{align}
where 
\begin{align}\pnlty_\radius(t) = \begin{cases}
(1 - t/\radius)^2 \ & \text{for }t < \radius\\
0 \ & \text{for } t \ge \radius
\end{cases}
\end{align}
This is reminiscent of the all-layer margin technique of~\citep{wei2019improved}, except we analyze a continuous loss function, whereas their technique only applies to the 0-1 loss. We provide the following generalization bound for $\ploss_{\radius}$:

\begin{lemma}\label{lem:ploss_gen}
	With probability $1 - \delta$ over the draw of the training sample, for all $W \in \R^{\dout \times \din}$ with $\normto{W^\top} \le A$, and all $\alpha > 0$, we have 
	\begin{align} \label{eq:ploss_gen:1}
	\E_{\dist}[\ploss_\radius(W, x, y)] \le (1 + 1/\alpha) \E_{\edist}[\ploss_\radius(W, x, y)] + \const (1 + \alpha)\ub \left (\frac{A^2 \log^3(nc)}{\radius^2 n} + \frac{\log 1/\delta + \log \log n}{n}\right) 
	\end{align}
	for some universal constant $\const > 0$. 
\end{lemma}

We prove the above lemma in Section~\ref{sec:ploss_gen_proof}. 
Our proof technique for obtaining Lemma~\ref{lem:max_sigma_bound} will be as follows: we first note that $\ploss_\radius$ is an upperbound on $\widebar{\loss}$. Second, we will upper bound $\ploss_{\radius}$ in terms of $\widebar{\loss}$ and the derivatives of $\loss$ on the training data. Combining these two bounds and optimizing over $\sigma$ will roughly give Lemma~\ref{lem:max_sigma_bound}. We have the following lemma bounding $\ploss_{\radius}$. 

\begin{lemma}\label{lem:ell_growth_bound}
	In the setting of Lemma~\ref{lem:max_sigma_bound}, suppose $\loss(\cdot, y)$ satisfies~\eqref{eq:exp_tail_cond} for all $h, y$. Then we have the upper bound
	\begin{align}
	\ploss_\radius(W, x, y) \le \loss(Wx, y) + \gnorm{D\loss(\cdot, y)[Wx]} \gnorm{x} \radius + \tr(D^2 \loss(\cdot, y)[Wx]) \frac{\exp(\radius \tau \gnorm{x}) - \radius \tau \gnorm{x} - 1}{\tau^2}
	\end{align}
\end{lemma}
We prove this lemma in Section~\ref{sec:ell_growth_bound_proof}.

\begin{proof}[Proof of Lemma~\ref{lem:max_sigma_bound}.]
	Our starting point is Lemma~\ref{lem:ploss_gen}. Our strategy will be to bound $\ploss_\radius$ in terms of the original loss $\loss$, and pick the best possible choice of $\radius$ for this bound. Define $\beta_j \triangleq \frac{\sqrt{\const} A\log^{3/2}(nc)}{\sqrt{n}}\exp(j)$ and let $\cS \triangleq \{\beta_j : j = 0, \ldots, J \}$ where we set $J \triangleq \lceil \log(\sqrt{n}) \rceil $. Our strategy will be to apply Lemma~\ref{lem:ploss_gen} for all choices of $\radius$ in $\cS$ and show that there is some choice of $\radius \in \cS$ giving~\eqref{eq:max_sigma_bound:1}. 
	
	We apply Lemma~\ref{lem:ploss_gen} for $\sigma = \beta_0, \ldots, \beta_{J}$ using probability $\delta/|\cS|$ and union bound over the failure probability. This allows us to conclude that with probability $1 - \delta$, for all $\normto{W^\top} \le A$ and $\beta \in \cS$, 	\begin{align} 
	\E_{\dist}[\ploss_\beta(W, x, y)] &\le (1 + 1/\alpha) \E_{\edist}[\ploss_\beta(W, x, y)] + \const (1 + \alpha)\ub \left (\frac{A^2 \log^3(nc)}{\beta^2 n} + \frac{\log |\cS|/\delta + \log \log n}{n}\right)\\
	&\le (1 + 1/\alpha) \E_{\edist}[\ploss_\beta(W, \cdot)] + \const (1 + \alpha)\ub \frac{A^2 \log^3(nc)}{\beta^2 n} + \const' (1 + \alpha)\ub \frac{\log 1/\delta + \log \log n}{n} \label{eq:max_sigma_bound:2}
	\end{align}
	In the last line, $\const'$ is a universal constant. We used the fact that $|\cS| \lesssim \log n$. Now using the fact that $\ploss_\radius$ upper bounds $\loss$ and applying Lemma~\ref{lem:max_sigma_bound}, we obtain from~\eqref{eq:max_sigma_bound:2} for all $\normto{W^\top} \le A$, $j = 0, \ldots, J$:
	\begin{align}
	\E_\dist[\loss(Wx, y)] \le (1 + 1/\alpha) \frac{1}{n} \sum_{i = 1}^n \loss(Wx, y) + \const_1G_{W,\alpha}(\beta_j) + \const_2 (1 + \alpha)\ub \frac{\log 1/\delta + \log \log n}{n}
	\end{align}
	where $\const_1, \const_2$ are universal constants and $G_{W,\alpha}(\beta)$ is defined in~\eqref{eq:max_sigma_bound:3}. (Note that $G_{W,\alpha}(\beta)$ depends on $W$ and the training data.) For a fixed choice of $W$, let $\beta^\star \triangleq \argmin_{\beta > 0} G_{W,\alpha}(\beta)$. First, if $\beta^\star \in [\beta_0, \beta_{J}]$, then by construction $\exists \widebar{\beta}\in \cS$ such that $\widebar{\beta} \in [\beta^\star/e, \beta^\star]$ where $e$ denotes the mathematical constant. For this choice of $j$, we have $G_{W,\alpha}(\widebar{\beta}) \le e^2 G_{W,\alpha}(\beta^\star)$ (as the third term in $G_{W,\alpha}(\beta)$ is the only decreasing term in $\beta$, and it differs by a factor of at most $e^2$ from $\widebar{\beta}$ to $\beta^\star$.)
	
	Now consider the case when $\beta^\star < \beta_0 = \frac{\sqrt{\const} A \log^{3/2}(nc)}{\sqrt{n}}$. In this case, $G_{W,\alpha}(\beta^\star) > B$, and so we trivially have~\eqref{eq:max_sigma_bound:1} since $\loss$ is upper bounded by $B$. 
	
	Finally, in the case when $\beta^\star > \beta_J$, we note that $G_{W,\alpha}(\beta_J) - G_{W,\alpha}(\beta^\star) \lesssim \frac{B}{n}$. This is again because only the third term in $G_{W,\alpha}(\beta)$ is decreasing in $\beta$, and for $\beta > \beta_J$, this term is at most $(1 + \alpha)B/n$. 
	
	Thus, for all choices of $\beta^\star$, we can conclude that 
	\begin{align}
	\E_\dist[\loss(Wx, y)] \le (1 + 1/\alpha) \E_{\edist}[\loss(Wx, y)] + \const_3G_{W,\alpha}(\beta^\star) + \const_4 (1 + \alpha) \ub \frac{\log 1/\delta + \log \log n}{n}
	\end{align}
	for universal constants $\const_3, \const_4 > 0$. This gives the desired result.
\end{proof}

It suffices to prove Lemmas~\ref{lem:ploss_gen} and~\ref{lem:ell_growth_bound}. 
\subsubsection{Proof of Lemma~\ref{lem:ploss_gen}} \label{sec:ploss_gen_proof}
To prove Lemma~\ref{lem:ploss_gen}, we must define the empirical Rademacher complexity of a class of functions. For $\cF$ a class of functions taking values in $\R$, the empirical Rademacher complexity is defined by 
\begin{align}
	\erad(\cF) = \E_{z_i} [\sup_{f \in \cF} z_i f(x_i, y_i)]
\end{align}
where $(x_i, y_i)$ are datapoints in the training sample and $z_i$ are drawn i.i.d. and uniformly from $\{-1, +1\}$. Furthermore, for some set $\cS$ (i.e., some function class), let $\cN_{\| \cdot \|}(\epsilon, \cS)$ be the covering number of $\cS$ in the metric induced by the norm $\| \cdot \|$ with error $\epsilon$. We will use the notation $L_2(\edist)$ to denote the following norm defined via the training sample: $\| f\|_{L_2(\edist)} \triangleq (\E_{x \sim \edist}[f(x)^2])^{1/2}$. $\cN_{L_2(\edist)}$ will then denote the covering number in norm $\| \cdot \|_{L_2(\edist)}$.  We will use the proof technique of~\citep{srebro2010smoothness}. 

We will require the following bound on how $\ploss_{\radius}$ changes when the weight matrix $W$ changes.
\begin{claim} \label{claim:ploss_smooth}
	For any $W, W'$ we have $(\ploss_{\radius}(W,x,y) - \ploss_{\radius}(W', x,y))^2 \le (\ploss_{\radius}(W,x,y) + \ploss_{\radius}(W', x,y))\frac{4\ub\gnorm{Wx - W'x}^2}{\gnorm{x}^2 \radius^2}$ 
\end{claim}

Define $\emplclass_{\radius}(r) = \{\ploss_\radius(W, \cdot, \cdot) : W \in \R^{\dout \times \din}, \normto{W^\top } \le A, \E_{\edist}[\ploss_\radius(W, x, y)] \le r\}$ to be the data-dependent\footnote{Note that $\emplclass_{\radius}(r)$ is data-dependent because the loss on the training data is required to be bounded by $r$.} class of loss functions with average empirical loss at most $r$. Next, we will require a certain covering number bound for $\{W x_i\}_{i = 1}^n$:
\begin{claim} \label{claim:two_one_cover}
	In the above setting, define the set $\cW(r) \triangleq \{W \in \R^{\dout \times \din} : \normto{W^\top} \le A, \E_{\edist}[\ploss_\radius(W, x, y)] \le r\}$. For any choice of $\epsilon > 0$, there exists a set of matrices $\widebar{\cW}_\epsilon \subset \cW(r)$ with cardinality bounded by $$\log |\widebar{\cW}_\epsilon| \le 1152 \lfloor \frac{4A^2}{\epsilon^2} \rfloor (\log_2(2 \lceil 16 A /\epsilon + 2 \rceil n + 1) + \log \dout)$$ satisfying the following: for all $W \in \cW(r)$ , there exists $\widebar{W}\in \widebar{\cW}_\epsilon$ such that 
	\begin{align}
	\frac{\gnorm{Wx_i - \widebar{W}x_i}}{\gnorm{x_i}} \le \epsilon \forall \ i = 1, \ldots, n
	\end{align}
\end{claim}

Applying Claims~\ref{claim:ploss_smooth} and~\ref{claim:two_one_cover} lets us bound the covering number of $\emplclass_{\radius}(r)$.
\begin{claim}\label{claim:local_cover_bound}
	In the above setting, we have the covering number bound 
	\begin{align} \label{eq:local_cover_bound:1}
	\log \cover_{L_2(\edist)}\left(\frac{\sqrt{8\ub r}}{\radius}\epsilon, \emplclass_\radius(r)\right) \le 1152 \lfloor \frac{4A^2}{{\epsilon}^2} \rfloor (\log_2(2 \lceil 16 A /{\epsilon} + 2 \rceil n + 1) + \log \dout)
	\end{align}
\end{claim}

\begin{proof}[Proof of Claim~\ref{claim:local_cover_bound}]
	Let $\widebar{\cW}_{\epsilon}$ be the set of matrices inducing the $\epsilon$ cover of $\{Wx_i/\gnorm{x_i}\}_{i = 1}^n$ whose cardinality is bounded in Claim~\ref{claim:two_one_cover}. For any $\widebar{W} \in \widebar{\cW}_{\epsilon}$, we can compute
	\begin{align}
	\|\ploss_\radius(W, \cdot, \cdot) - \ploss_\radius(\widebar{W}, \cdot, \cdot)\|_{L_2(\edist)} &= \sqrt{\frac{1}{n}\sum_{i = 1}^n (\ploss_\radius(W, x_i, y_i) - \ploss_\radius(\widebar{W}, x_i, y_i))^2}\\
	&\le \sqrt{\frac{1}{n} \sum_{i = 1}^n (\ploss_{\radius}(W,x_i,y_i) + \ploss_{\radius}(\widebar{W},x_i,y_i))\frac{4B\gnorm{Wx_i - W'x_i}^2}{\gnorm{x_i}^2 \radius^2}} \tag{by Claim~\ref{claim:ploss_smooth}}\\
	&\le \frac{2}{\radius}\sqrt{\frac{B}{n} \sum_{i=1}^n (\ploss_\radius(W, x_i, y_i) + \ploss_\radius(\widebar{W}, x_i, y_i))} \max_i \frac{\gnorm{W x_i - \widebar{W} x_i}}{\gnorm{x_i}}\\
	&\le \frac{\sqrt{8Br}}{\radius} \epsilon
	\end{align}
	Thus, using $\{\ploss_\radius(\widebar{W}, \cdot, \cdot) : \widebar{W} \in \widebar{\cW}_{\epsilon}\}$ lets us conclude~\eqref{eq:local_cover_bound:1}. 
\end{proof}

This translates into the following Rademacher complexity bound for $\emplclass_{\radius}(r)$: 
\begin{claim} \label{lem:local_rad_bound}
	In the setting of Claim~\ref{claim:local_cover_bound}, we have
	\begin{align}
	\erad(\emplclass_{\radius}(r)) \le  \frac{3500 A \sqrt{ \ub r \log^3(35 nc)}}{\radius \sqrt{n}}
	\end{align}
\end{claim}

\begin{proof}[Proof of Claim~\ref{lem:local_rad_bound}.]
	We apply Dudley's entropy integral using the covering number bound in Claim~\ref{claim:local_cover_bound}. This mirrors the calculation used to prove Lemma 2.2 in~\citep{srebro2010smoothness}. From Lemma A.1 of~\citep{srebro2010smoothness}, we have
	\begin{align}\label{eq:local_rad_bound:1}
	\erad(\emplclass_\radius(r)) \le \inf_{\alpha > 0} \left (4 \alpha + 10 \int_{\alpha}^{\sqrt{\ub r}} \sqrt{\frac{\log \cover_{L_2(\edist)} (\epsilon, \emplclass_\radius(r))}{n}} d\epsilon\right)
	\end{align}
	Now we perform a change of variables $\epsilon = \frac{\sqrt{8\ub r}}{\radius} \epsilon'$, after which~\eqref{eq:local_rad_bound:1} becomes
	\begin{align}
	\erad(\emplclass_\radius(r)) &\le \inf_{\alpha > 0} \left (4 \alpha + 10 \frac{\sqrt{8\ub r}}{\radius} \int_{\frac{\radius\alpha}{\sqrt{8\ub r}}}^{\radius/\sqrt{8}} \sqrt{\frac{\log \cover_{L_2(\edist)} \left(\frac{\sqrt{8 \ub r}}{\radius}\epsilon', \emplclass_\radius(r)\right)}{n}} d\epsilon'\right)\\
	&\le \inf_{\alpha > 0} \left (4 \alpha + 680 \frac{A\sqrt{8\ub r}}{\radius} \int_{\frac{\radius\alpha}{\sqrt{8 \ub r}}}^{\min \{2A, \radius/\sqrt{8}\}} \frac{\sqrt{\log_2(2 \lceil 16 A /{\epsilon'} + 2 \rceil n + 1) + \log \dout}}{\epsilon'\sqrt{n}} d\epsilon' \right) \label{eq:local_rad_bound:2}
	\end{align}
	To change the upper limit of the integral from $\sigma/\sqrt{8}$ to $\min\{2A, \sigma/\sqrt{8}\}$, we used the fact that we only need to integrate $\epsilon'$ to $2A$, because for $\epsilon' > 2A$ the log covering number is 0 by Claim~\ref{claim:local_cover_bound}. 
	Now we plug in $\alpha = \frac{A \sqrt{8 \ub r}}{\radius \sqrt{n}}$ into~\eqref{eq:local_rad_bound:2} and use the fact that we only integrate over $\epsilon' > A/\sqrt{n}$ to obtain (after simplification):
	\begin{align}
	\erad(\emplclass_\radius(r)) &\le \frac{A \sqrt{128 \ub r}}{\radius \sqrt{n}} + 680 \frac{A\sqrt{8\ub r}}{\radius} \int_{\frac{A}{\sqrt{n}}}^{2A} \frac{\sqrt{3 \log 35 n \dout}}{\epsilon'\sqrt{n}} d\epsilon'\\
	&\le \frac{A \sqrt{128 \ub r}}{\radius \sqrt{n}} + 680 \frac{A\sqrt{24\ub r \log (35nc)}}{\radius \sqrt{n}}\log(2\sqrt{n})\\
	&\le \frac{3500 A \sqrt{ \ub r \log^3(35 nc)}}{\radius \sqrt{n}}
	\end{align}	
\end{proof}

Using Claim~\ref{lem:local_rad_bound}, we can complete the proof of Lemma~\ref{lem:ploss_gen} using the technique of~\citep{srebro2010smoothness}, which is essentially local Rademacher complexity~\citep{bousquet2002concentration}. \begin{proof} [Proof of Lemma~\ref{lem:ploss_gen}.]
	Define $\psi(r) \triangleq \frac{3500A \sqrt{\ub r \log^3 (35nc)}}{\radius \sqrt{n}}$. By Claim~\ref{lem:local_rad_bound}, we have $\erad(\emplclass_\radius(r)) \le \psi(r)$. 
	
	Thus, we can apply the steps from the proof of Theorem 1 in~\citep{srebro2010smoothness} (which invokes Theorem 6.1 of~\citep{bousquet2002concentration}), we have with probability $1 - \delta$, for all $W \in \R^{\dout \times \din},\normto{W^\top} \le A$ 
	\begin{align}
	\E_{\dist}[\ploss_\radius(W,  x, y)] &\le \E_{\edist}[\ploss_\radius(W,  x, y)] + 106r^\star + \frac{48 \ub }{n} (\log 1/\delta + \log \log n) \nonumber\\
	& + \sqrt{\E_{\edist}[\ploss_\radius(W,  x, y)] (8 r^\star + \frac{4 \ub}{n} (\log 1/\delta + \log \log n))}\nonumber
	\end{align}
	By the AM-GM inequality, for all $\alpha > 0$, we have
	\begin{align}
	\E_{\dist}[\ploss_\radius(W,  x, y)] &\le (1 + 1/2\alpha) \E_{\edist}[\ploss_\radius(W,  x, y)] + r^\star(106 + 4\alpha) + \frac{(48 + 2 \alpha) \ub}{n} (\log 1/\delta + \log \log n)
	\end{align}
	where $r^\star$ satisfies $\psi(r^\star) = r^\star$. Now using the fact that $r^\star \lesssim \frac{BA^2 \log^3(nc)}{\radius^2 n}$, we obtain~\eqref{eq:ploss_gen:1} for some $\const > 0$.
\end{proof}

\begin{proof}[Proof of Claim~\ref{claim:ploss_smooth}.]
	Let $\delta^\star$ be the optimal perturbation for $W$ and $x$, i.e. $$\delta^\star \triangleq \argmax_{\delta \in \R^\dout} \pnlty_\radius(\gnorm{\delta}) \widebar{\loss}(Wx + \delta \gnorm{x}, y)$$
	We construct a perturbation $\delta'$ for the objective of $\ploss_{\radius}(W', x, y)$ as follows: define $\delta' \triangleq \delta^\star + \frac{Wx - W'x}{\gnorm{x}}$. It follows that 
	\begin{align}
	\ploss_{\radius}(W', x, y) &\ge \pnlty_\radius(\gnorm{\delta'}) \widebar{\loss}(W'x + \delta' \gnorm{x}, y)\\
	&=  \pnlty_\radius(\gnorm{\delta'}) \widebar{\loss}(Wx + \delta^\star \gnorm{x}, y) \tag{by construction of $\delta'$}\\
	&\ge \pnlty_\radius \left(\gnorm{\delta^\star} + \frac{\gnorm{Wx - W'x}}{\gnorm{x}}\right) \widebar{\loss}(Wx + \delta^\star \gnorm{x}, y) \tag{by triangle inequality}\\
	&\ge \left(\pnlty_\radius(\gnorm{\delta^\star}) - \frac{2\gnorm{Wx - W'x}}{\gnorm{x}\radius} \sqrt{\pnlty_\radius(\gnorm{\delta^\star})}\right) \widebar{\loss}(Wx + \delta^\star \gnorm{x}, y) \tag{using Claim~\ref{claim:pnlty_smooth}}
	\end{align}
	Thus, rearranging and using the fact that $\ploss_\radius(W, x, y) = \pnlty_\radius(\gnorm{\delta^\star}) \widebar{\loss}(Wx + \delta^\star \gnorm{x}, y)$, we obtain
	\begin{align}
	\ploss_{\radius}(W', x, y) - \ploss_\radius(W, x, y) &\ge -\frac{2\gnorm{Wx - W'x}}{\gnorm{x}\radius} \sqrt{\pnlty_\radius(\gnorm{\delta^\star})} \widebar{\loss}(Wx + \delta^\star \gnorm{x}, y)\\
	&\ge -\sqrt{\ub}\frac{2\gnorm{Wx - W'x}}{\gnorm{x}\radius} \sqrt{\ploss_\radius(W, x, y)} \tag{using the upper bound $(\widebar{\loss}(Wx + \delta^\star \gnorm{x}, y))^{1/2} \le \sqrt{\ub}$}
	\end{align}
	Using the same reasoning, we can also obtain
	\begin{align}
	\ploss_{\radius}(W, x, y) - \ploss_\radius(W', x, y) &\ge -\sqrt{\ub}\frac{2\gnorm{Wx - W'x}}{\gnorm{x}\radius} \sqrt{\ploss_\radius(W', x, y)}
	\end{align}
	It thus follows that 
	\begin{align}
	|\ploss_{\radius}(W', x, y) - \ploss_\radius(W, x, y)| \le \sqrt{\ub}\frac{2\gnorm{Wx - W'x}}{\gnorm{x}\radius} \max\{ \sqrt{\ploss_\radius(W, x, y)}, \sqrt{\ploss_\radius(W', x, y)}\}
	\end{align}
	Squaring both sides gives the desired result. 
	\end{proof}

\begin{proof}[Proof of Claim~\ref{claim:two_one_cover}]
	We first construct a set of matrices $\widetilde{\cW}_\epsilon \subset \R^{\dout \times \din}$ satisfying for all $W \in \R^{\dout \times \din}$ with $\normto{W^\top} \le A$, there exists $\widebar{W} \in \widetilde{\cW}_\epsilon$ with 
	\begin{align}
	\frac{\gnorm{Wx_i - \widebar{W}x_i}}{\gnorm{x_i}} \le \epsilon \forall \ i = 1, \ldots, n
	\end{align}
	We first note that when $\epsilon \ge A$, this set only needs cardinality 1, as we simply take $\widetilde{\cW}_\epsilon$ to only have the all 0's matrix. First, consider a $\epsilon/2$-cover in $\gnorm{\cdot}$-norm of the set $\{v \in \R^\dout : \|v\|_1 \le A\}$, which we denote by $\widebar{\cV}$. By classical results, such a cover exists with log cardinality $4A^2/\epsilon^2 \log (c + 1)$. 
	
	Next, by Theorem 4 of~\cite{zhang2002covering}, for all choices of $\epsilon', a > 0$, there exists a set $\widebar{\cU}(\epsilon', a) \subset \{u \in \R^\din : \gnorm{u} \le a\}$ such that for any $u \in \R^\din$, $\gnorm{u} \le a$, there exists $\widebar{u} \in \widebar{\cU}(\epsilon', a)$ such that $\frac{|u^\top x_i - \widebar{u}^\top x_i |}{\gnorm{x_i}} \le \epsilon' \ \forall i = 1, \ldots, n$. Furthermore, the cardinality of this set satisfies the bound 
	\begin{align}
	\widebar{\cU}(\epsilon', a) \le 144 \frac{a^2}{{\epsilon'}^2} \log_2 (2\lceil 4a/\epsilon' + 2\rceil n + 1)
	\end{align}
	
	Now for any $v \in \widebar{\cV}$, we add to our cover the set $\widetilde{\cW}_\epsilon(v) \triangleq \{\widebar{W} \in \R^{\dout \times \din} : \onevec_j^\top \widebar{W} \in \widebar{\cU}(\epsilon \sqrt{|v_j|/4\|v\|_1}, |v_j|)\}$. We have 
	\begin{align}
	\log |\widetilde{\cW}_\epsilon(v)| &\le \sum_{j = 1}^{\dout} \log |\widebar{\cU}(\epsilon \sqrt{|v_j|/4\|v\|_1}, |v_j|)|\\
	&\le \sum_{j = 1}^{\dout} 576 \frac{\|v\|_1 |v_i|}{\epsilon^2} \log_2(2\lceil 8\sqrt{\|v\|_1 |v_i|}/\epsilon + 2 \rceil n + 1 )\\
	&\le 576 \frac{\|v\|_1^2}{\epsilon^2} \log_2(2\lceil 8\|v\|_1/\epsilon + 2 \rceil n + 1 )\\
	&\le 576 \frac{A^2}{\epsilon^2} \log_2(2 \lceil 8 A /\epsilon + 2 \rceil n + 1)
	\end{align}   
	Furthermore, setting $\widetilde{\cW}_\epsilon \triangleq \cup_{v \in \widebar{\cV}} \widetilde{\cW}_\epsilon(v)$, we thus have 
	\begin{align}
	\log |\widetilde{\cW}_\epsilon| &\le \log |\widebar{\cV}| + 576 \frac{A^2}{\epsilon^2} \log_2(2 \lceil 8 A /\epsilon + 2 \rceil n + 1)\\
	&\le 1152 \lfloor \frac{A^2}{\epsilon^2} \rfloor (\log_2(2 \lceil 8 A /\epsilon + 2 \rceil n + 1) + \log \dout)
	\end{align}
	To obtain the last line, we use the fact that $\widetilde{W}_\epsilon$ has cardinality 0 when $\epsilon \ge A$.
	It remains to show that $\widetilde{W}$ satisfies the desired error properties. For any $W$ satisfying $\normto{W^\top} \le A$, there exists $\widebar{v} \in \widebar{\cV}$ satisfying 
	\begin{align} \label{eq:two_one_cover:2}
	\gnorm{\widebar{v} - \{\gnorm{\onevec_j^\top W}\}_{j = 1}^\dout} \le \epsilon/2
	\end{align} 
	Furthermore, by construction there exists $\widebar{W} \in \widetilde{\cW}_\epsilon(\widebar{v})$ satisfying 
	\begin{align}
	\label{eq:two_one_cover:1}
	\frac{\left|\onevec_j^\top \widebar{W} x_i - \frac{\widebar{v}_j}{\gnorm{\onevec_j^\top W}} \onevec_j^\top Wx_i\right|}{\gnorm{x_i}} \le \frac{\epsilon}{2} \sqrt{\frac{|\widebar{v}_j|}{\|\widebar{v}\|_1}} \ \forall j = 1, \ldots, c, \ i = 1, \ldots, n
	\end{align}
	It follows that for all $i = 1, \ldots, n$, we have 
	\begin{align}
	\frac{\gnorm{Wx_i - \widebar{W}x_i}}{\gnorm{x_i}} &= \frac{\sqrt{\sum_j (\onevec_j^\top (\widebar{W}x_i - W x_i))^2}}{\gnorm{x_i}}\\
	&\le \frac{\sqrt{\sum_j \left(\onevec_j^\top \widebar{W} x_i - \frac{\widebar{v}_j}{\gnorm{\onevec_j^\top W}} \onevec_j^\top Wx_i\right)^2}}{\gnorm{x_i}} + \frac{\sqrt{\sum_j \left(\frac{\widebar{v}_j}{\gnorm{\onevec_j^\top W}} - 1\right )^2 (\onevec_j^\top W x_i)^2}}{\gnorm{x_i}} \tag{by triangle inequality}\\
	&\le \epsilon/2 + \sqrt{\sum_j (\widebar{v}_j - \gnorm{\onevec_j^\top W})^2 \frac{(\onevec_j^\top W x_i)^2}{\gnorm{\onevec_j^\top W}^2 \gnorm{x_i}^2}} \tag{applying~\eqref{eq:two_one_cover:1}}\\
	&\le \epsilon/2 + \sqrt{\sum_j (\widebar{v}_j - \gnorm{\onevec_j^\top W})^2} \tag{since $\frac{(\onevec_j^\top W x_i)^2}{\gnorm{\onevec_j^\top W}^2 \gnorm{x_i}^2} \le 1$}\\
	&\le \epsilon \tag{by~\eqref{eq:two_one_cover:2}}
	\end{align}
	To conclude the statement of the lemma, we note that for each element $\widebar{W}\in \widetilde{\cW}_{\epsilon/2}$, we can add to $\widebar{\cW}_{\epsilon}$ a single $\widebar{W}' \in \cW(r)$ satisfying
	\begin{align}
	\frac{\gnorm{\widebar{W}x_i - \widebar{W}'x_i}}{\gnorm{x_i}} \le \epsilon/2 \forall \ i = 1, \ldots, n
	\end{align} 
	Then $\widebar{\cW}_{\epsilon}$ will be the desired cover with cardinality bounded by
	\begin{align}
	\log |\widebar{\cW}_{\epsilon}| \le \log |\widetilde{\cW}_{\epsilon/2}| \le 1152 \left \lfloor \frac{4A^2}{\epsilon^2} \right \rfloor (\log_2(2 \lceil 16 A /\epsilon + 2 \rceil n + 1) + \log \dout)
	\end{align}
\end{proof}

\subsubsection{Proof of Lemma~\ref{lem:ell_growth_bound}} \label{sec:ell_growth_bound_proof}

We will rely on the following bound on the change of a function satisfying~\eqref{eq:exp_tail_cond}. 
\begin{claim}\label{claim:exp_tail_loss}
	Suppose the loss function $\loss : \R^\dout \to \R$ is convex and satisfies $\gnorm{\deriv (\tr \circ \deriv^2 \loss )[h]} \le \tau \tr (\deriv^2 \loss [h])$ for all $h$ and some $\tau > 0$. Then for all $h, h'$, we have 
	\begin{align}
	\loss(h') \le \loss(h) + \gnorm{\deriv\loss(h)}\gnorm{h' - h} + \tr(\deriv^2 \loss[h]) \frac{\exp(\tau \gnorm{h' - h}) - \tau\gnorm{h' - h} - 1}{\tau^2}
	\end{align}
	
\end{claim}
\begin{proof}
	The proof of this claim mirrors the proof of Proposition 1 in~\citep{bach2010self}.
\end{proof}

Now we complete the proof of Lemma~\ref{lem:ell_growth_bound}. 
\begin{proof} [Proof of Lemma~\ref{lem:ell_growth_bound}.]
	We can calculate
	\begin{align}
	\ploss_\radius(W, x, y) &= \max_{\delta \in \R^\dout} s_\radius(\gnorm{\delta}) \widebar{\loss}(Wx + \delta \gnorm{x}, y)\\
	&\le \max_{\gnorm{\delta} \le \radius} \loss(Wx + \delta \gnorm{x}, y) \tag{since $s_\radius \le 1$ and $s_\radius(\gnorm{\delta}) = 0$ when $\gnorm{\delta} > \radius$, and $\widebar{\loss} \le \loss$}\\
	&\le \max_{\gnorm{\delta} \le \radius} \loss(Wx, y) + \gnorm{D \loss(\cdot, y)[Wx]} \gnorm{\delta} {\gnorm{x}} + \tr(D^2 \loss(\cdot, y)[Wx]) \frac{\exp(\gnorm{\delta} \tau \gnorm{x}) - \gnorm{\delta} \tau \gnorm{x} - 1}{\tau^2} \tag{by Claim~\ref{claim:exp_tail_loss}}\\
	&\le  \loss(Wx, y) + \gnorm{D\loss(\cdot, y)[Wx]} \gnorm{x} \radius + \tr(D^2 \loss(\cdot, y)[Wx]) \frac{\exp(\radius \tau \gnorm{x}) - \radius \tau \gnorm{x} - 1}{\tau^2} \tag{since the previous equation is increasing in $\gnorm{\delta}$}
	\end{align}
\end{proof}
\subsection{Proof of Lemma~\ref{lem:G_upper_bound}} \label{sec:G_upper_bound_proof}

We will rely on the following statement regarding the optimum of a function which shows up in our proof for Lemma~\ref{lem:G_upper_bound}. 
\begin{claim}\label{claim:exp_opt_bound}
	Suppose that $G : \R \rightarrow  \R$ is a function of the form 
	\begin{align}
	G(\beta) = a_1 (\exp(\beta) - \beta - 1) + a_2/\beta^2 
	\end{align}
	Then we have 
	\begin{align}
	\min_{\beta > 0} G(\beta) \lesssim \sqrt{a_1 a_2} + \frac{a_2}{\log^2 \frac{a_2}{a_1} + 1}
	\end{align}
\end{claim}

\begin{proof}[Proof of Lemma~\ref{lem:G_upper_bound}]
	We drop the $W, \alpha$ dependency in the notation for simplicity. Define 
	\begin{align}
	G_1(\beta) &\triangleq \beta (1 + 1/\alpha)\gradm\kappa \\
	G_2(\beta) &\triangleq  (1 + 1/\alpha) \frac{\hessm(\exp(\tau \beta \kappa ) - \tau\beta\kappa - 1)}{\tau^2}\\
	G_3(\beta) &\triangleq (1 + \alpha) \frac{\ub A^2 \log^3(nc)}{\beta^2 n}
	\end{align}
	We first claim that 
	\begin{align} \label{eq:G_upper_bound}
	\min_{\beta > 0} G(\beta) < \min_{\beta > 0}  (G_1(\beta) + G_3(\beta)) + \min_{\beta > 0} (G_2(\beta) + G_3(\beta))
	\end{align} 
	To see this, let $\beta^\star, \beta_1, \beta_2$ be the minimizers for $G, G_1 + G_3, G_2 + G_3$, respectively, and assume without loss of generality that $\beta_2 \ge \beta_1$. Then we have 
	\begin{align}
	G(\beta^\star) \le G(\beta_1) &= G_1(\beta_1) + G_2(\beta_1) + G_3(\beta_1)\\
	&\le G_1(\beta_1) + G_2(\beta_2) + G_3(\beta_1) \tag{since $G_2$ is an increasing function}\\
	&< G_1(\beta_1) + G_3(\beta_1) + G_2(\beta_2) + G_3(\beta_2)
	\end{align} 
	which gives the desired statement from the definitions of $\beta_1, \beta_2$. Thus, it suffices to minimize $G_1 + G_3$, $G_2 + G_3$ separately. 
	
	To bound the minimum of $G_1 + G_3$, we observe that to minimize any function of the form $a_1 \beta + a_2/\beta^2$, we can set $\beta = (a_1)^{-1/3} a_2^{1/3}$. Applying this to $G_1 + G_3$ gives 
	\begin{align}
	\min_{\beta} G_1(\beta) + G_3(\beta) \lesssim \frac{(1 + 1/\alpha)^{2/3}(1 + \alpha)^{1/3} B^{1/3} \log(nc)}{n^{1/3}}(A\mu \kappa )^{2/3}
	\end{align}
	
	Next, we bound the minimum of $G_2 + G_3$. As the function $\exp(a)- a - 1$ is increasing in $a$, we have
	\begin{align}
	G_2(\beta) &\le (1 + 1/\alpha) \hessm \frac{\exp(\tau \beta \kappa ) - \tau\beta \kappa  - 1}{\tau^2}
	\end{align}
	where $\hessm$ is defined in~\eqref{eq:hessm_def}. 
	Let $\widebar{G}_2(\beta)$ denote the right-hand side of the above equation. Now we can invoke Claim~\ref{claim:exp_opt_bound} on the variable $\tau \beta \kappa$ to conclude that 
	\begin{align}
	\min_{\beta > 0} G_2(\beta) + G_3(\beta) < \min_{\beta > 0} \widebar{G}_2 + G_3(\beta) \lesssim \\
	\frac{A\kappa \sqrt{(1 + \alpha)(1 + 1/\alpha) \ub \log^3(nc)}}{\sqrt{n}} \sqrt{\hessm} +  (1 + \alpha)\frac{\ub A^2 \log^3(nc) \kappa^2 \tau^2}{n\left(\log^2 \left(\frac{(1 + \alpha)\ub A^2 \log^3(nc) \kappa^2 \tau^4}{n(1 + 1/\alpha) \hessm} \right) + 1\right)}
	\end{align}
	Finally, invoking~\eqref{eq:G_upper_bound} and applying the definition of $\theta$ gives the desired result.
\end{proof}

\begin{proof}[Proof of Claim~\ref{claim:exp_opt_bound}.]
	First consider the case when $a_1 > a_2$. In this case, set ${\beta'}^2 = \sqrt{a_2/a_1} < 1$. As $\exp(b) - b - 1 \le b^2$ for $0 \le b < 1$, we have in this case $G(\beta') \le \sqrt{a_1 a_2}$. 
	
	Otherwise, consider the case when $a_2 > a_1$ but $\frac{a_2}{a_1} \le \log^2 \frac{a_2}{a_1}$, $\log^2 \frac{a_2}{a_1} < 1$, or $\log \frac{a_2}{a_1} \le 2\log \log^2 \frac{a_2}{a_1} + 2$. If any of these three equations hold, then $a_2/a_1$ is upper bounded by some universal constant, in which case setting $\beta' = 1$ immediately gives $G(\beta') \lesssim \sqrt{a_1 a_2}$. 
	
	Otherwise, consider the case when  $\frac{a_2}{a_1} > \log^2 \frac{a_2}{a_1}$, $\log^2 \frac{a_2}{a_1} \ge 1$, and $\log^2 \frac{a_2}{a_1} > 2\log \log^2 \frac{a_2}{a_1} + 2$ all hold. In this case, we set $\beta' = \log \left( \frac{a_2}{a_1 \log^2 \frac{a_2}{a_1}}\right)$. Note that $\beta' > 0$ and $\beta' = \log \frac{a_2}{a_1} - \log \log^2 \frac{a_2}{a_1} > \frac{1}{2} \log \frac{a_2}{a_1} + 1$ by our conditions on $a_2/a_1$. Then we have 
	\begin{align}
	G(\beta') &\le a_1 \exp(\beta) + a_2/\beta^2 \\
	&\le \frac{a_2}{\log^2 \frac{a_2}{a_1}} + \frac{4a_2}{(\log \frac{a_2}{a_1} + 1)^2}\\
	&\lesssim \frac{a_2}{\log^2 \frac{a_2}{a_1} + 1}
	\end{align}
	Combining the three cases gives the desired claim. 
\end{proof}

\subsection{Additional Proofs of Helper Lemmas} \label{sec:proof_additional}
We first provide the proof of Lemma~\ref{lem:cross_entropy_exp_tailed}. We use $\probs$ to denote the vector of probabilities predicted by the cross entropy loss, formally defined in Section~\ref{sec:helpers}. We will rely on the derivatives of the cross entropy loss computed in Section~\ref{sec:helpers}. 
\begin{proof}[Proof of Lemma~\ref{lem:cross_entropy_exp_tailed}]
	We first compute $\tr (\deriv^2 \ce_y [h]) = \sum_i p_i - p_i^2$, by Section~\ref{sec:helpers}. Next, we have $\deriv \tr (\deriv^2 \ce_y)[h] = \sum_i (1 - 2p_i) p_i (\onevec_i - \probs)$. Now by the convexity of $\gnorm{\cdot }$, as $\sum_i p_i = 1$, we have
	\begin{align}
	\gnorm{\sum_i (1 - 2p_i) p_i (\onevec_i - \probs)} &\le \sum_i p_i |1 - 2p_i| \gnorm{\onevec_i - \probs}\\
	&\le \sum_i p_i \gnorm{\onevec_i - \probs} \tag{since $|1 - 2p_i| \le 1$}
	\end{align}
	Now we have 
	\begin{align}
	\gnorm{\onevec_i - \probs} &= \sqrt{(1 - p_i)^2 + \sum_{j \ne i} p_j^2}\\
	&\le \sqrt{(1 - p_i)^2 + (\sum_{j \ne i} p_i)^2}\\
	&\le (1 - p_i)\sqrt{2}
	\end{align}
	Plugging this back into the previous equation, we obtain 
	\begin{align}
	\gnorm{\sum_i (1 - 2p_i) p_i (\onevec_i - \probs)} \le  \sum_i p_i \gnorm{\onevec_i - \probs} \le \sqrt{2} \left(\sum_i p_i - p_i^2\right)
	\end{align}
	This gives the desired result.
\end{proof}

Next, the following statement is useful for Claim~\ref{claim:ploss_smooth}.
\begin{claim} \label{claim:pnlty_smooth}
	In the setting of Claim~\ref{claim:ploss_smooth}, for any $a, b > 0$, we have 
	\begin{align}(\pnlty_\radius(a + b) - \pnlty_\radius(a))^2 \le 4\frac{1}{\radius^2}\pnlty_\radius(a) b^2 \label{eq:pnlty_smooth:1}
	\end{align}
\end{claim}
\begin{proof}
	First, in the case where $a + b \ge \radius, a \ge \radius$, we have $\pnlty_\radius(a + b) = \pnlty_\radius(a) = 0$ so the inequality trivially holds.
	
	Second, in the case where $a + b \ge \radius, a < \radius$, we have $\pnlty_\radius(a + b) = 0$, so the LHS of~\eqref{eq:pnlty_smooth:1} is simply $\pnlty_\radius(a)^2$. Now we note that $b \ge \radius - a$, so $b^2/\radius^2 \ge (1 - a/\radius)^2 = \pnlty_\radius(a)$. Thus, we have $\pnlty_\radius(a)^2 \le \pnlty_\radius(a) b^2/\radius^2$, so~\eqref{eq:pnlty_smooth:1} follows. 
	
	Third, in the case where $a + b < \radius, a < \radius$, we have 
	\begin{align}
		\pnlty_\radius(a) - \pnlty_\radius(a + b) &= (2 - \frac{2a + b}{\radius})(\frac{b}{\radius}) \tag{expanding the expression for $\pnlty_\radius$}
	\end{align}
	Squaring both sides, we obtain
	\begin{align}
		(\pnlty_\radius(a) - \pnlty_\radius(a + b))^2 &= (2 - \frac{2a + b}{\radius})^2\frac{b^2}{\radius^2}\\
		&\le 4(1 - a/\radius)^2 \frac{b^2}{\radius^2}\\
		&\le 4\pnlty_\radius(a) \frac{b^2}{\radius^2}
	\end{align}
\end{proof}

\subsection{Discussion of Bound}
\label{sec:theory:app_discussion}
Consider the case where the empirical distribution is only supported on $\dout' \ll \dout$ tightly clustered classes and all the datapoints $x_i$ have norm 1. Suppose that the norms of the rows of $W$ are balanced and concentrated on the $\dout'$ classes in the empirical sample. Let $\probs$ denote the softmax probability vector defined in~\eqref{eq:dist_def}. Consider weights $W$ which are well aligned with the data, so that $1  - \probs(Wx_i)_{y_i} \approx \exp(-\|\onevec_{y_i}^\top W\|_2) \approx \exp(-\|W\|_F/\sqrt{\dout'})$, for every training example $(x_i, y_i)$ (this is possible because the softmax probability vector is exponential-tailed).

In this case, by the expression for the Hessian of cross entropy loss (see Section~\ref{sec:helpers}), we would have $\log(1/\hessm(W)) \approx \E_{\edist}[-\log (1  - \probs(Wx)_y)] \gtrsim \frac{\|W\|_F}{\sqrt{\dout'}}$. On the other hand, we also have $\normto{W^\top} = \sqrt{\dout'} \|W\|_F$. Thus, the third term in the bound becomes $O\left(\frac{\ub (\dout')^2 \log^3(n \dout)}{n} \right)$.

\section{Additional Implementation Details} \label{sec:implementation_app}
We implement our code in PyTorch, basing our LSTM implementation on the following code: \url{https://github.com/salesforce/awd-lstm-lm}. We base our Transformer-XL implementation on the following code: \url{https://github.com/kimiyoung/transformer-xl}. Code for downloading and pre-processing the datasets which we use are also contained in these repositories. We run our code on NVIDIA TitanXp GPUs. We provide detailed descriptions of the algorithms we implement below. 
\subsection{Implementing Our Explicit Regularizer} \label{sec:exp_reg_imp}
Because of the large output dimensionality of language modeling tasks, we cannot compute $\regexp{}$ exactly and instead approximate it by sampling. In this section, we describe how to implement the sampling procedure for variants of the popular cross entropy loss defined in~\eqref{eq:ce_def}. Recall that we use the notation $\ce_y(v) \triangleq \ce(v, y) = -\log \sftmax(v)_y$, where $\sftmax(v)$ denotes the softmax distribution computed from $v$. We leverage the following relationship between the first and second order derivatives of the loss: 
\begin{align} \label{eq:hess_sample}
	\lhess(x) = \E_{\hat{y} \sim \sftmax(\net(x))}[ \deriv \ce_{\hat{y}}[\net(x)]^\top \deriv \ce_{\hat{y}}[\net(x)] ]
\end{align}
This relationship formally proved in Claim~\ref{claim:hess_sample}. 
Note in particular that our regularizer $\regexp{}$ does not depend on the true label $y$ since the loss Hessian $\lhess$ is independent of $y$. 
Algorithm~\ref{alg:estimate_reg_exp} leverages this relationship to compute an unbiased estimator for $\regexp{}$ based on~\eqref{eq:hess_sample} by sampling a label $\hat{y} \sim \sftmax(\net(x))$ and computing the loss Jacobian for label $\hat{y}$ instead of the full Hessian matrix $\lhess$.\footnote{When computing the gradient update, we do not differentiate through the sampling probabilities $\sftmax(\net(x))$. Thus, though our \textit{loss} estimate is unbiased, our estimate of $\nabla_W \regexp(\net, x)$ is biased. This does not appear to matter in practice.} We formally prove the correctness of Algorithm~\ref{alg:estimate_reg_exp} below. 

\begin{claim}
	Algorithm~\ref{alg:estimate_reg_exp} gives an unbiased estimate of $\regexp{}$ defined in~\eqref{eq:reg_ours}.
\end{claim} 
\begin{proof}
	Note that 
	\begin{align}
		\E_{\hat{y} \sim \sftmax(\net(x))}[\hat{J}_i \diag(\hidden{i}(x)^{\odot 2}) \hat{J}_i^\top] &= \E_{\hat{y} \sim \sftmax(\net(x))}\left [\matip{\hat{J}_i^\top \hat{J}_i}{ \diag(\hidden{i}(x)^{\odot 2})}\right]\\
		&= \E_{\hat{y} \sim \sftmax(\net(x))}\left [\matip{\jac{i}(x)^\top \deriv \ce_{\hat{y}}[\net(x)]^\top \deriv \ce_{\hat{y}}[\net(x)]  \jac{i}(x)}{ \diag(\hidden{i}(x)^{\odot 2})}\right]\\
		&= \matip{\jac{i}(x)^\top \lhess(x) \jac{i}(x)}{ \diag(\hidden{i}(x)^{\odot 2})} \tag{by~\eqref{eq:hess_sample}}
	\end{align}
	Summing over $i$ gives the desired result. 
\end{proof}

Algorithm~\ref{alg:estimate_reg_exp} admits a straightforward extension to the adaptive softmax loss~\citep{grave2017efficient} which computes the derivative for the loss with respect to a sampled cluster label and sampled word within the cluster.

\begin{algorithm}[tb]
	\caption{Unbiased estimate of $\regexp{}$ for cross-entropy loss.}
	\label{alg:estimate_reg_exp}
	\begin{algorithmic}
		\STATE {\bfseries Input:} data $x$. 
		\STATE Sample $\hat{y} \sim \sftmax (\net(x))$. 
		\STATE Initialize $r = 0$. 
		\FOR{layers $i$}
		\STATE Compute $\hat{J}_i = \deriv \ce_{\hat{y}} \circ \neti{i}[\hidden{i}(x)]$.
		\STATE Update $r = r + \hat{J}_i \diag(\hidden{i}(x)^{\odot 2}) \hat{J}_i^\top$. 
		\ENDFOR
		\STATE {\bfseries Return} $r$. 
	\end{algorithmic}
\end{algorithm}

\subsection{Implementing Figure~\ref{fig:imp_reg_only} Experiment}
Algorithm~\ref{alg:drop_k_our_noise} describes more formally how to implement $\drop{k}$ with injection of noise $\impours{}$, which was plotted in Figure~\ref{fig:imp_reg_only}. 
\begin{algorithm}[t]
	\caption{Procedure for injecting our update noise into $\drop{k}$ updates.}
	\label{alg:drop_k_our_noise}
	\begin{algorithmic}
		\STATE {\bfseries Input:} minibatch $\{x_i\}_{i = 1}^{m}$, number of dropout noise samples $k$, dropout probability $\droppr$.
		\STATE Sample noise $\noise_{ij}$ for $i \in [m], j \in [k]$.
		\STATE Compute $g = \nabla_W \left(\frac{1}{m} \sum_{i = 1}^{m} \helldrop{k}(\net, x_i, \{\noise_{ij}\}_{j = 1}^k) \right)$.
		\STATE Sample noise $\noise_{mk + 1}$ with coordinates independently and uniformly distributed in $\{-1, +1\}$.
		\STATE Update $g = g + \sqrt{\frac{\droppr}{\droppr - 1}}\sqrt{1 - \frac{1}{k}} \impours{}(\net, x, \noise_{mk + 1})$.\\
		$\triangleright$ Use $g$ for optimization algorithm.
	\end{algorithmic}
\end{algorithm}

\subsection{Using Identity Instead of Loss Hessian} 
It is non-trivial to implement this experiment described in Section~\ref{sec:hessian}, as the dimensionality of the output is large and naively computing the regularizer $\matip{\jac{i}^\top  \jac{i}}{\diag(\hidden{i}^{\odot 2})}$ requires computing the output Jacobian $\jac{i}$ exactly. To circumvent this issue, we use sampling. Letting $\noise_i$ be a random vector whose coordinates are independently and uniformly sampled from $\{-1, +1\}$, we have
\begin{align}
	\E[(\eta_i \odot \hidden{i})^\top \jac{i}^\top \jac{i} (\eta_i \odot \hidden{i})] = \matip{\jac{i}^\top  \jac{i}}{\diag(\hidden{i}^{\odot 2})}
\end{align}
Now to compute the value $\jac{i} (\eta_i \odot \hidden{i})$ we use the method for computing Jacobian vector products described here: \url{https://j-towns.github.io/2017/06/12/A-new-trick.html}. 
\section{Additional Experimental Results} \label{sec:experimental_result_app}
\subsection{Additional Results for Section~\ref{sec:replace_dropout}}
For all experiments in Section~\ref{sec:replace_dropout}, we use SGD with learning rate of 30 with gradient clipping with a threshold of 0.35, and default $\ell_2$-regularization of 1.2e-6 (unless we specify that we tuned this parameter). These parameters are the defaults from the awd-lstm-lm repository.\footnote{\url{https://github.com/salesforce/awd-lstm-lm}} For Penn Treebank, we use a batch size of 20 training for 150 epochs and for WikiText-2, we use a batch size of 40 training for 100 epochs. We use these same base settings for the experiments in Section~\ref{sec:hessian}. Combining our explicit and implicit regularizers adds 7-8 times runtime overhead per iteration compared to dropout. 

We use a coefficient of $\lambda_1 = 2/3$ for our explicit regularizer. For our implicit regularizer, for the experiments corresponding to Figure~\ref{fig:imp_reg_only}, we provide implementation details in Algorithm~\ref{alg:drop_k_our_noise}. For the experiments which combine our explicit and implicit regularizers, we implement $\impours{}$ by sampling $\noise$ with coordinates independently and uniformly distributed in $\{-1, +1\}$, and computing $\impours{}(\net, x, \noise)$ with coefficient $\lambda_2 = \sqrt{2/3}$ in Algorithm~\ref{alg:update_rule}. This is meant to match the dropout probability of 0.4. 

In Tables~\ref{tab:PTB_results} and~\ref{tab:wiki2_results}, we summarize our results across the experiments in Section~\ref{sec:replace_dropout}. 
\begin{table}
	\centering
	\caption{A summary of the validation perplexities of our regularizers and various baselines on Penn Treebank. Our regularizers match their dropout counterparts.}	\label{tab:PTB_results}
	\begin{tabular}{c c c} 
		& Training Method & Best Val. Ppl.\\
		\cline{1-3}
		\multirow{5}{*} {Baselines} & $\drop{1}$ & 73.76\\
				& $\drop{8}$ & 83.82\\
		& $\drop{32}$ & 89.12\\
		\cline{2-3}
		& No regularization & 122.16\\
				& Best $\ell_2$ reg. & 112.04 \\
		\cline{1-3}
		\multirow{4}{*}{Our regularizers} 
		& $\regexp{}$~\eqref{eq:reg_ours}   & 84.52\\
		& $\regexp{}$, 8 samples & 84.27\\ 
		\cline{2-3}
		& Algorithm~\ref{alg:drop_k_our_noise}, $k =8$ & 74.49\\
		\cline{2-3}
		& $\regexp{}$ and $\impours{}$~\eqref{eq:imp_ours} & 72.99
	\end{tabular}
\end{table}

\begin{table}
	\centering
	\caption{A summary of the validation perplexities of our regularizers and various baselines on Wikitext-2. Our regularizers match their dropout counterparts.}	\label{tab:wiki2_results}
	\begin{tabular}{c c c} 
		& Training Method & Best Val. Ppl.\\
		\cline{1-3}
		\multirow{5}{*} {Baselines} & $\drop{1}$ & 90.97\\
				& $\drop{8}$ & 95.91\\
		& $\drop{32}$ & 98.10\\
		\cline{2-3}
		& No regularization & 144.12\\
				& Best $\ell_2$ reg. & 137.50\\
		\cline{1-3}
		\multirow{4}{*}{Our regularizers} 
		& $\regexp{}$~\eqref{eq:reg_ours}   & 92.26\\
		& $\regexp{}$ (8 samples) & 94.73\\ 
		\cline{2-3}
		& Algorithm~\ref{alg:drop_k_our_noise}, $k =8$ & 90.87\\
		\cline{2-3}
		& $\regexp{}$ and $\impours{}$~\eqref{eq:imp_ours} & 84.57
	\end{tabular}
	
	\end{table}

\subsection{Additional Results for Section~\ref{sec:add_exp}}
For our Transformer-XL experiments, we use the hyperparameters for the base model on WikiText-103 contained in the following repository: \url{https://github.com/kimiyoung/transformer-xl/}. We use an explicit regularization coefficient of $\lambda = 0.11$ to match the dropout probability $\droppr = 0.1$. Our explicit regularizer takes 3 times longer per iteration than dropout. 

We also examine the implicit regularization effect of dropout on the QRNN architecture for WikiText-103. We use a batch size of 15 with the Adam optimizer and an initial learning rate of 5e-4 for all our runs. We chose these parameters to fit the $\drop{k}$ updates in memory. The other hyperparameters are set to their defaults in the awd-lstm repository. We chose to use QRNN as they are faster to train than LSTMs~\citep{bradbury2016quasi,merity2018analysis}. Table~\ref{tab:wiki-103-qrnn} demonstrates that the implicit regularization effect does not appear on the WikiText-103 dataset, which also matches our observations for the Transformer-XL architecture on this same dataset. This suggests that the dataset size, and not architecture, influences whether the implicit regularization effect appears. 

\begin{table}
	\centering
	\caption{Experimental results on the full WikiText-103 dataset for QRNN architecture.}	\label{tab:wiki-103-qrnn}
	\begin{tabular}{c c} 
		Training Method & Best Val. Ppl.\\
		\cline{1-2}
		$\drop{1}$   & 34.24\\
		$\drop{2}$ &  33.35\\
		$\drop{4}$ & 32.74\\
		$\drop{8}$ & 32.78
	\end{tabular}
	
\end{table}

\section{Useful Properties of Cross-Entropy Loss} \label{sec:helpers}
Recall that we defined the cross entropy loss $\ce_y$ with label $y$ by
\begin{align} \label{eq:ce_def}
\ce_y(Wx) \triangleq -\log \frac{\exp((Wx)_y)}{\sum_{y'} \exp((Wx)_{y'})}
\end{align}
The derivatives of the cross-entropy loss are given as follows: 
\begin{align}
\deriv_{\var{v}} \ce_y[v] =  \probs(v) - \onevec_y
\end{align}
where $\probs(v)$ is the softmax probability vector given by 
\begin{align} \label{eq:dist_def}
	(\probs(v))_{y'} = \frac{\exp(v_{y'})}{\sum_{y''} \exp(v_{y''})}
\end{align}
Furthermore, it also holds that 
\begin{align}
	\deriv^2_{\var{v}} \ce_y[v] = \deriv^2_{\var{v}} \probs[v] = \diag(\probs(v)) - \probs(v) \probs(v)^\top
\end{align}
Furthermore, we have the following relationship between the first and second derivatives of the loss: 
\begin{claim}\label{claim:hess_sample}
	$\deriv^2 \ce_y[v] = \E_{\hat{y}\sim \probs(v)}[\deriv\ce_{\hat{y}}[v]^\top\deriv\ce_{\hat{y}}[v]]$.
\end{claim}
\begin{proof}
	We have 
	\begin{align}
	\E_{\hat{y}\sim \probs(v)}[\deriv\ce_{\hat{y}}[v]^\top \deriv\ce_{\hat{y}}[v]] &= \E_{\hat{y}\sim \probs(v)} [(\probs(v) -  \onevec_{\hat{y}})(\probs(v) -  \onevec_{\hat{y}})^\top]
	\end{align}
	Note that the expectation of $\onevec_{\hat{y}}$ for $\hat{y}\sim \probs(v)$ is simply $\probs(v)$. Thus, the right hand side simplifies to the desired statement. 
\end{proof}

\appendix

\end{document}